\documentclass[11pt]{article}

\usepackage[dvipsnames, table]{xcolor}

\usepackage[pdftex]{graphicx}

\usepackage{epstopdf}

\usepackage[cmex10]{amsmath}
\usepackage{amsfonts, amssymb, amsthm, mathtools}
\usepackage{empheq}

\usepackage[caption=false, format=hang]{subfig}

\usepackage{multirow}   		
\usepackage{booktabs}   		
\usepackage{array}
\newcolumntype{+}{>{\global\let\currentrowstyle\relax}}
\newcolumntype{^}{>{\currentrowstyle}}

\usepackage{microtype}
\usepackage{algorithm}

\newcommand{\minimize}[2]{\ensuremath{\underset{\substack{{#1}}}{\operatorname{minimize}}\;\;#2}}

\newcommand{\scal}[2]{{\left\langle{{#1}\mid{#2}}\right\rangle}}
\newcommand{\menge}[2]{\big\{{#1}~\big |~{#2}\big\}} 
\newcommand{\HH}{\ensuremath{{\mathcal H}}}

\newcommand{\RR}{\ensuremath{\mathbb{R}}}
\newcommand{\RP}{\ensuremath{\left[0,+\infty\right[}}

\newcommand{\RPP}{\ensuremath{\left]0,+\infty\right[}}

\newcommand{\RX}{\ensuremath{\left]-\infty,+\infty\right]}}

\newcommand{\NN}{\ensuremath{\mathbb N}}

\newcommand{\bell}{\boldsymbol{\ell}}

\newcommand{\pinf}{\ensuremath{{+\infty}}}

\newcommand{\prox}{\ensuremath{\operatorname{prox}}}

\newcommand{\epi}{\ensuremath{\operatorname{epi}}}
\newcommand{\subto}{\ensuremath{\operatorname{s.t.}}}

\DeclarePairedDelimiter{\parens}{(}{)}

\newtheorem{theorem}{Theorem}[section]

\newtheorem{remark}[theorem]{Remark}
\newtheorem{proposition}[theorem]{Proposition}

\newcommand{\argmax}{\operatorname*{argmax\;}}

\newcommand{\citet}[1]{\cite{#1}}
\newcommand{\citep}[1]{\cite{#1}}

\begin{document}

\title{A Proximal Approach for Sparse Multiclass SVM\thanks{This work was supported by the CNRS IMAG'in OPTIMISME project}}

\author{G. Chierchia\thanks{G. Chierchia (Corresponding author) and B. Pesquet-Popescu are with  T\'el\'ecom ParisTech/Institut T\'el\'ecom, LTCI, UMR CNRS 5141, 75014 Paris, France (e-mail: first.last@telecom-paristech.fr).}\and Nelly Pustelnik\thanks{N. Pustelnik is with the Laboratoire de Physique de l'ENS Lyon, CNRS UMR 5672, F69007 Lyon, France. Phone: +33 4 72 72 86 49, E-mail: \texttt{nelly.pustelnik@ens-lyon.fr}.}\and  Jean-Christophe Pesquet\thanks{J.-C. Pesquet is  with the Universit{\'e} Paris-Est, LIGM, CNRS-UMR 8049, 77454 Marne-la-Vall{\'e}e Cedex 2, France. Phone: +33 1 60 95 77 39, E-mail: \texttt{jean-christophe.pesquet@univ-paris-est.fr}.} \and B. Pesquet-Popescu$^*$}

\maketitle

\begin{abstract}%
	Sparsity-inducing penalties are useful tools to design multiclass support vector machines (SVMs). In this paper, we propose a convex optimization approach for efficiently and exactly solving the multiclass SVM learning problem involving a sparse regularization and the multiclass hinge loss formulated by \citet{Crammer_K_2001_j-mach-learn-res_algorithmic_imk}. We provide two algorithms: the first one dealing with the hinge loss as a penalty term, and the other one addressing the case when the hinge loss is enforced through a constraint. The related convex optimization problems can be efficiently solved thanks to the flexibility offered by recent primal-dual proximal algorithms and epigraphical splitting techniques. Experiments carried out on several datasets demonstrate the interest of considering the exact expression of the hinge loss rather than a smooth approximation. The efficiency of the proposed algorithms w.r.t.\ several state-of-the-art methods is also assessed through comparisons of execution times.
\end{abstract}
	
\section{Introduction}
Support vector machines (SVMs) have gained much popularity in solving large-scale classification problems. As a matter of fact, many applications considered in the literature deal with a large amount of training data or a huge (even infinite) number of classes \citep{Martin_D_2005_j-naasp_speech_rbm, Tsochantaridis_I_2005_j-mach-learn-res_large_mms, Huang_F_2006_p-ieee-cvpr_large_sls, Laptev_2008_p-ieee-cvpr_learning_rha, Joachims_T_2009_j-mach-learn_cutting_pts}. Consequently, the major difficulty encountered in this kind of applications stems from the computational cost. The SVM learning problem is classically solved by using standard Lagrangian duality techniques \citep{Cortes_C_1995_j-mach-learn_supp_vect_net, Crammer_K_2001_j-mach-learn-res_algorithmic_imk}. This approach brings in several advantages, such as the kernel trick \citep{Aizerman_1964_kernel_trick}, or the possibility to break the problem down into a sequence of smaller ones \citep{Platt1995,Blondel_icpr2014}. Some works also proposed to approximate the dual problem using cutting plane approaches, in order to address scenarios with thousands or even an infinite number of classes \citep{Tsochantaridis_I_2005_j-mach-learn-res_large_mms, Joachims_T_2009_j-mach-learn_cutting_pts}.

In some applications, however, only a small number of training data is available. This is undoubtedly true in medical contexts, where the goal is to classify a patient as being ``healthy", ``contaminated'', or ``infected'', but the verified cases of infected patients might be just a few. In such applications, the lack of training data may lead to the so-called \textit{overfitting} problem,  eventually leading to a prediction which is too strongly tailored to the particularities of the training set and poorly generalizes to new data. 

A common solution to prevent overfitting consists of introducing a sparsity-inducing regularization in order to perform an implicit \emph{feature selection} that gets rid of irrelevant or noisy features. In this respect, the $\boldsymbol{\ell}_1$-norm and, more generally, the $\boldsymbol{\ell}_{1,p}$-norm regularization have attracted much attention over the past decade \citep{Krishnapuram_2005_j-pami-multi_logit, Duchi2009_boost-sparsity, Yuan2010_comparisonL1, Rakotomamonjy2011_j-tnn_lplq_penalty_multitask, Bach_F_2012_j-ftml_optimization_sip, Rosasco2013_j-mlr_nonparam_sparse_reg, Tuia2014_j-tgrs_auto_feat_learn_svm, Villa_2014_prox_methods_lasso}. However, when a sparse regularization is introduced, the dual approach does no longer  yield a simple formulation. 
Therefore, SVMs with sparse regularization lead to a nonsmooth convex optimization problem which is challenging. The main objective of this paper is to \emph{exactly} and \emph{efficiently} solve the multiclass SVM learning problem for convex regularizations. To this end, we propose two algorithms based on a primal-dual proximal method \citep{Vu_B_2011_j-acm_spl_adm,Condat_L_2012} and a novel epigraphical splitting technique \citep{Chierchia_G_2014_SIVP_epigraphical_ppt}. In addition to more detailed theoretical developments, this paper extends our preliminary work \citep{Chierchia2014_ICASSP_sparse_SVM} by providing a new algorithm, and a larger number of experiments including comparisons with state-of-the-art methods for different types of database.

\subsection{Related work} 
The use of sparse regularization in SVMs was firstly proposed in the context of binary classification. The idea traces back to the work by~\citet{Bradley1998_ICML_featureselection}, who demonstrated that the $\bell_1$-norm regularization can effectively perform ``feature selection'' by shrinking small coefficients to zero. Other forms of regularization have also been studied, such as the $\bell_0$-norm~\citep{Weston2002_j-mach-learn_L0SVM}, the $\bell_p$-norm with $p > 0$ \citep{Liu2007_LpSVM}, the $\bell_\infty$-norm \citep{Zou2008_stat_LinfSVM}, and the combination of $\bell_0$-$\bell_1$ norms \citep{Liu2007_L0L1SVM} or $\bell_1$-$\bell_2$ norms \citep{Wang06_stat_L1L2SVM}. A different solution was proposed by~\citet{Tan_M_2010_p-icml_learning_ssvm}, who reformulated the SVM learning problem by using an indicator vector (its components being either equal to 0 or 1) to model the active features, and solved the resulting combinatorial problem by convex relaxation using a cutting-plane algorithm. More recently, \citet{Laporte2014_j-ieee-tnnls_sparseSVM} proposed an accelerated algorithm for $\bell_1$-regularized SVMs involving the \emph{square hinge} loss. They also proposed a procedure for handling nonconvex regularization (using the reweighted $\bell_1$-minimization scheme by \citet{Candes_2008_j-four-anal-appl_enhancing_srl}), showing that nonconvex penalties lead to similar prediction quality while using less features than convex ones.

Binary SVMs can be turned into multiclass classifiers by a variety of strategies, such as the \emph{one-vs-all} approach \citep{Cortes_C_1995_j-mach-learn_supp_vect_net, Rifkin2004_one-vs-all}. While these techniques provide a simple and powerful framework, they cannot capture the correlations between different classes, since they break a multiclass problem into multiple \emph{independent} binary problems. \citet{Crammer_K_2001_j-mach-learn-res_algorithmic_imk} therefore proposed a direct formulation of multiclass SVMs by generalizing the notion of margins used in the binary case. A natural idea thus consists of equipping muticlass SVMs with sparse regularization. A simple example is the $\bell_1$-regularized multiclass SVM, which can be addressed by linear programming techniques \citep{Wang2007_stat_L1MSVM}. In multiclass problems however, feature selection becomes more complex than in the binary case, since multiple discriminating functions need to be estimated, each one with its own set of important features. For this reason, mixed-norm regularization has recently attracted much interest due to its ability to impose group sparsity \citep{Yuan2006_stat_modelselection, Meier2008_stat_grouplasso, Duchi2009_boost-sparsity, Obozinski2010_joint-selection}. In the context of multiclass SVMs, \citet{Zhang2008_variable-selection} proposed to deal with the $\bell_{1,\infty}$-norm regularization by reformulating the SVM learning problem in terms of linear programming. However, they validated their method on small-size problems, indicating that the linear reformulation may be inefficient for larger-size ones. More recently, \citet{Blondel_2013_j-mach-learn_block_coord_SMV} proposed an algorithm to handle $\bell_{1,2}$-regularized SVMs involving a smooth loss function. While their method is efficient and can handle other convex regularizations, it does not solve rigorously the multiclass SVM learning problem, possibly leading to performance limitations.

\subsection{Contributions} The algorithmic solutions proposed in the literature to deal with sparse multiclass SVMs are either cutting-plane methods~\citep{Tan_M_2010_p-icml_learning_ssvm}, proximal algorithms \citep{Laporte2014_j-ieee-tnnls_sparseSVM, Blondel_2013_j-mach-learn_block_coord_SMV}, or linear programming techniques \citep{Wang2007_stat_L1MSVM,Zhang2008_variable-selection}. However, both cutting-plane methods and proximal algorithms have been employed to find an approximate solution, while linear programming techniques may not scale well to large datasets. In this paper, we propose a novel approach based on proximal tools and recent epigraphical splitting techniques \citep{Chierchia_G_2014_SIVP_epigraphical_ppt}, which allow us to exactly solve the sparse multiclass SVM learning problem through an efficient primal-dual proximal method \citep{Vu_B_2011_j-acm_spl_adm,Condat_L_2012}.

\subsection{Outline} 
The paper is organized as follows. In Section~\ref{s:pa}, we formulate the multiclass SVM problem with sparse regularization, in Section~\ref{s:algo} we provide the proximal tools needed to solve the proposed problem, and in Section~\ref{s:nr} we evaluate our approach on three standard datasets and compare it to the methods proposed by~\citet{Blondel_2013_j-mach-learn_block_coord_SMV}, \citet{Laporte2014_j-ieee-tnnls_sparseSVM}, \citet{Zhang2008_variable-selection}, and~\citet{Krishnapuram_2005_j-pami-multi_logit}.

\subsection{Notation} 
$\Gamma_0(\RR^X)$ denotes the set of proper, lower semicontinuous, convex functions from the 
Euclidean space $\RR^X$ to $\RX$. 
The epigraph of $\psi\in \Gamma_0(\RR^X)$ is the nonempty closed convex subset of $\RR^X\times \RR$ defined as $\epi \psi = \menge{(y,\zeta) \in \RR^X\times \RR}{\psi(y)\le \zeta}$. For every $x \in \RR^X$, the subdifferential of $\psi$ at $x$ is $\partial \psi(x) = \menge{u\in\RR^X}{(\forall y\in\RR^X)\; \scal{y-x}{u}+\psi(x)\leq \psi(y)}$. Let $C$ be a nonempty closed convex subset of $\RR^X$, then $\iota_C$ is the indicator function of $C$, equal to $0$ on $C$ and $\pinf$ otherwise.


\section{Sparse Multiclass SVM}\label{s:pa}
A multiclass classifier can be modeled as a function $d\colon \RR^N \to \{1,\dots,K\}$ that predicts the class $k \in \{1,\dots,K\}$ associated to a given observation $u\in \RR^N$ (e.g.\ a signal, an image or a graph). This predictor relies on $K$ different \emph{discriminating functions} $D_k \colon \RR^N \mapsto \RR$ which, for every $k \in \{1,\dots,K\}$, measure the likelihood that an observation belongs to the class $k$. Consequently, the predictor selects the class that best matches an observation, i.e.
\begin{equation*}
d(u) \in \argmax_{k\in \{1,\dots,K\}} D_k(u).
\end{equation*}
In supervised learning, the discriminating functions are built from a set of $L$ input-output pairs
\begin{equation*}
S = \big\{ (u_\ell,z_\ell) \in \RR^N \times \{1,\dots,K\} \;|\; \ell=\{1,\dots,L\}\big\},
\end{equation*}
and they are assumed to be linear in some feature representation of inputs \citep{Cover_1965_j-ieee-ec_geom_stat_prop}. The latter assumption leads to
the following form of the discriminating functions:
\begin{equation}\label{eq:D_k}
D_k(u) = \phi(u)^\top x^{(k)} + b^{(k)},
\end{equation}
where $\phi \colon \RR^N \mapsto \RR^M$ denotes a mapping from the input space onto an arbitrary feature space, 
and $(x^{(k)}, b^{(k)})_{1\le k\le K}$ denote the parameters to be estimated. For convenience, we concatenate the latter ones into a single vector ${\rm x}\in \RR^{(M+1)K}$
\begin{equation*}
{\rm x} = 
\left[
\begin{aligned}
&x^{(1)}\\
&b^{(1)}\\
&\;\;\,\vdots\\
&x^{(K)}\\
&b^{(K)}\\
\end{aligned}
\right]
\begin{aligned}
&\left.\vphantom{\begin{aligned}x^{(1)}\\b^{(1)}\end{aligned}}\right\} {\rm x}^{(1)}\\
&\vphantom{\vdots}\\
&\left.\vphantom{\begin{aligned}x^{(K)}\\b^{(K)}\end{aligned}}\right\} {\rm x}^{(K)}
\end{aligned}
\end{equation*}
and we define the function $\varphi \colon \RR^N \mapsto \RR^{M+1}$ as
\begin{equation*}
\varphi(u) = \left[\phi(u)^\top \;\; 1 \right]^\top,
\end{equation*}
so that \eqref{eq:D_k} can be shortened to $D_k(u)= \varphi(u)^\top {\rm x}^{(k)}$.

\subsection{Background}
The objective of learning consists of finding the vector ${\rm x}$ such that, for every $\ell \in \{1,\dots,L\}$, the input-output pair $(u_\ell,z_\ell) \in S$ is correctly predicted by the classifier, i.e.,
\begin{equation*}
z_\ell = \argmax_{k\in \{1,\dots,K\}} \varphi(u_\ell)^\top {\rm x}^{(k)}.
\end{equation*}
By the definition of $\operatorname{argmax}$, the above equality holds if\footnote{To simplify the notation, we shorten $k \in \{1,\dots,K\}\setminus\{z_\ell\}$ to $k\ne z_\ell$.}
\begin{equation*}
(\forall\ell \in \{1,...,L\})\quad \max_{k\ne z_\ell}\; \varphi(u_\ell)^\top ( {\rm x}^{(k)} - {\rm x}^{(z_\ell)} ) < 0,
\end{equation*}
or, equivalently,
\begin{equation}\label{eq:ideal_cond2}
(\forall\ell \in \{1,...,L\})\quad \max_{k\ne z_\ell}\; \varphi(u_\ell)^\top ( {\rm x}^{(k)} - {\rm x}^{(z_\ell)} ) \le -\mu_\ell,
\end{equation}
where, for every $\ell \in \{1,\dots,L\}$, $\mu_\ell$ is a positive scalar. Unfortunately, this constraint has no practical interest for learning purposes, as it becomes infeasible when the training set is not fully separable. Multiclass SVMs overcome this issue by introducing the notion of \emph{soft margins}, which consists of adding a vector of slack variables $\xi = (\xi^{(\ell)})_{1\le\ell\le L}$ into~\eqref{eq:ideal_cond2}:
\begin{equation}\label{eq:svm_constraint}
\left\{
\begin{aligned}
\!\!&(\forall\ell \in \{1,...,L\})\; \max_{k\ne z_\ell}\; \varphi(u_\ell)^\top ( {\rm x}^{(k)} - {\rm x}^{(z_\ell)} ) \le \xi^{(\ell)}-\mu_\ell,\\
\!\!&(\forall\ell \in \{1,...,L\})\; \xi^{(\ell)} \ge 0,
\end{aligned}
\right.
\end{equation}
The multiclass SVM learning problem is thus obtained by adding a quadratic regularization \citep{Crammer_K_2001_j-mach-learn-res_algorithmic_imk}, yielding\footnote{Note that the regularization does not involve the offsets $(b^{(k)})_{1\le k\le K}$.}
\begin{multline}\label{eq:slack_SVM}
\minimize{({\rm x},\xi)\in \RR^{(M+1)K}\times \RR^L}\; \sum_{k=1}^K\|x^{(k)}\|_2^2 + \lambda\sum_{\ell=1}^L \xi^{(\ell)} \quad\subto\\
\left\{
\begin{aligned}
\!\!&(\forall\ell \in \{1,...,L\})\quad \max_{k\ne z_\ell}\; \varphi(u_\ell)^\top ( {\rm x}^{(k)} - {\rm x}^{(z_\ell)} ) \le \xi^{(\ell)}-\mu_\ell,\\
\!\!&(\forall\ell \in \{1,...,L\})\quad \xi^{(\ell)} \ge 0,
\end{aligned}
\right.
\end{multline}
where $\lambda \in \RPP$.
Note that the linear penalty on the slack variables allows us to minimize the violation of  constraint~\eqref{eq:ideal_cond2}. By using standard convex analysis \citep{Boyd_S_2004_book_con_o}, the above problem can be equivalently rewritten without slack variables as
\begin{equation}\label{eq:SVM_hinge}
\operatorname*{minimize}_{{\rm x}\in \RR^{(M+1)K}}\; \sum_{k=1}^K\|x^{(k)}\|_2^2 + 
\lambda\sum_{\ell=1}^L \max\big\{0, \mu_\ell+\max_{k\ne z_\ell}\; \varphi(u_\ell)^\top ( {\rm x}^{(k)} - {\rm x}^{(z_\ell)} )\big\}.
\end{equation}
Hereabove, the second term is called \emph{hinge loss} when $\mu_\ell\equiv 1$.

\subsection{Proposed approach}
We extend Problem~\eqref{eq:SVM_hinge} by replacing the squared $\boldsymbol{\ell}_2$-norm regularization with a generic function $g \in \Gamma_0(\RR^{(M+1)K})$. Moreover, we rewrite the hinge loss in an equivalent form by introducing, for every $\ell \in \{1,\dots,L\}$, the linear operator $T_\ell \colon \RR^{(M+1)K} \mapsto \RR^{K}$ defined as
\begin{equation*}
\big(\forall{\rm x}\in\RR^{(M+1)K}\big)\quad T_\ell \, {\rm x} = \big[ \, \varphi(u_\ell)^\top ( {\rm x}^{(k)} - {\rm x}^{(z_\ell)} ) \,\big]_{1\le k \le K},
\end{equation*}
the vector $r_\ell = (r_\ell^{(k)})_{1\le k\le K}\in \RR^K$ defined as
\begin{equation*}
(\forall k\in\{1,\dots,K\})\qquad r_\ell^{(k)} = 
\begin{cases}
0, & \quad\textrm{if $k= z_\ell$},\\
\mu_\ell, & \quad\textrm{otherwise},
\end{cases}
\end{equation*}
and the function $h_\ell\colon\RR^{K}\mapsto \RR$ defined, for every $y^{(\ell)} = (y^{(\ell,k)})_{1\le k\le K} \in \RR^{K}$, as
\begin{equation}\label{eq:fun_h}
h_\ell(y^{(\ell)}) = \max_{1 \le k\le K} y^{(\ell,k)} + r_\ell^{(k)},
\end{equation}
so that the following holds
\begin{equation*}
h_\ell(T_\ell {\rm x}) = \max\big\{0, \mu_\ell+\max_{k\ne z_\ell}\; \varphi(u_\ell)^\top ( {\rm x}^{(k)} - {\rm x}^{(z_\ell)} )\big\}.
\end{equation*}
We aim at solving the following convex optimization problems:
\begin{itemize}
\item \textsl{regularized formulation}
\begin{equation}\label{eq:regularized_SVM_hinge}
\minimize{{\rm x} \in \RR^{(M+1)K}}\; g({\rm x}) + \lambda \sum_{\ell=1}^L h_\ell(T_\ell \,{\rm x}),
\end{equation}

\item \textsl{constrained formulation}
\begin{equation}\label{eq:constrained_SVM_hinge}
\minimize{{\rm x} \in \RR^{(M+1)K}}\; g({\rm x})  \quad\subto \quad \sum_{\ell=1}^L h_\ell(T_\ell \,{\rm x}) \le \eta,
\end{equation}
\end{itemize}
where $\lambda$ and $\eta$ are positive constants. Note that, by Lagrangian duality, the above formulations are equivalent for some specific values of $\eta$ and $\lambda$. The interest of considering the constrained formulation lies in the fact that $\eta$ may be easier to set, since it is directly related to the properties of the training data.

As mentioned in the introduction, the regularization term $g$ is chosen so as to promote some form of sparsity. A popular example is the $\boldsymbol{\ell}_1$-norm, as it ensures that the solution will have a number of coefficients exactly equal to zero, depending on the strength of the regularization \citep{Bach_F_2012_j-ftml_optimization_sip}. Another example is given by the mixed $\bell_{1,p}$-norm. For every ${\rm x} \in \RR^{(M+1)K}$, let us assume that, for each $k \in \{1,\dots,K\}$, the vector ${\rm x}^{(k)} \in \RR^{M+1}$ is block-decomposed as follows: 
\begin{equation*}
{\rm x}^{(k)} = \Big[\underbrace{\left(x^{(k,1)}\right)^\top}_{\text{size}\,M_1} \quad\dots\quad \underbrace{\left(x^{(k,B)}\right)^\top}_{\text{size}\,M_B} \quad b^{(k)}
\Big]^\top,
\end{equation*}
with $M_1 + \dots + M_B = M$. We define the $\bell_{1,p}$-norm as
\begin{equation*}
g({\rm x}) 
= \sum_{k=1}^K \sum_{b=1}^{B} \|x^{(k,b)}\|_p.
\end{equation*}
The mixed-norm regularization is known to induce \textsl{block-sparsity}: the solution is partitioned into groups and the components of each group are ideally either all zeros or all non-zeros. In this context, the exponent values $p=2$ or $p = +\infty$ are the most popular choices. In particular, the $\bell_{1,\infty}$-norm tends to favor solutions with few nonzero groups having components of similar magnitude.

\section{Optimization method}\label{s:algo}
The resolution of Problems~\eqref{eq:regularized_SVM_hinge} and \eqref{eq:constrained_SVM_hinge} requires an efficient algorithm for dealing with nonsmooth functions and hard constraints. In the convex optimization literature, proximal algorithms constitute one of the most efficient approaches to deal with nonsmooth problems \citep{Combettes_P_2010_inbook_proximal_smsp, Combettes_P_2011_j-svva_pri_dsa, Bach_F_2012_j-ftml_optimization_sip, Parikh_N_2014_j-fto_proximal_a, Komodakis_N_2014}. The key tool in these methods is the \emph{proximity operator} \citep{Moreau_J_1965_bsmf_Proximite_eddueh}, defined for a function $\psi\in\Gamma_0(\HH)$ as
\begin{equation*}
(\forall u \in \HH)\qquad
\prox_\psi(u)= \operatorname*{argmin}_{v\in\HH}\; \frac12 \|v-u\|^2 +
\psi(v).
\end{equation*}
The proximity operator can be interpreted as a sort of subgradient step for the function $\psi$, as $p = \prox_\psi(y)$ is uniquely defined through the inclusion $y - p \in \partial \psi(p)$. In addition, it reverts to the projection onto a closed convex set $C\subset\HH$ in the case when $\psi = \iota_C$, in the sense that 
\begin{equation}\label{e:projprox}
(\forall u \in \HH)\qquad\prox_{\iota_C}(u) = P_C(u) = \operatorname*{argmin}_{v\in C}\; \frac12 \|v-u\|^2.
\end{equation}

Proximal algorithms work by iterating a sequence of steps in which the proximity operators of the functions involved in the minimization are evaluated at each iteration. An efficient computation of these operators is thus essential to design fast algorithms for  solving Problems~\eqref{eq:regularized_SVM_hinge}-\eqref{eq:constrained_SVM_hinge}. In the next sections, we will present two different approaches based on a Forward-Backward based Primal-Dual method (FBPD) 
\citep{Vu_B_2011_j-acm_spl_adm, Condat_L_2012, Chambolle_A_2010_first_opdacpai, Combettes_2014_forward_backward_primal_dual, Combettes2014_variable_metric_FB}, which we have selected among the large panel of proximal algorithms for its simplicity to deal with large-size linear operators. 

\subsection{Regularized formulation}
Problem~\eqref{eq:regularized_SVM_hinge} fits nicely into the framework provided by FBPD algorithm, since the proximity operators of both $g$ and $(h_\ell)_{1\le\ell\le L}$ can be efficiently computed. Indeed, $\prox_g$ has a closed form for several norms and mixed norms \citep{Parikh_N_2014_j-fto_proximal_a,Combettes_P_2010_inbook_proximal_smsp}, while $(\prox_{h_\ell})_{1\le\ell\le L}$ can be computed through the projection onto the standard simplex, as described in Proposition~\ref{prop:simp}. The projection onto the simplex can be efficiently computed with the method proposed by \citet{Condat2014_fast_proj_L1}.

\begin{proposition} \label{prop:simp} 
For every $\ell\in \{1,\ldots, L\}$,
\begin{equation}\label{eq:prox_max}
(\forall { y}^{(\ell)}\in \RR^K)\quad \prox_{\lambda h_\ell} \big(y^{(\ell)}\big) = y^{(\ell)} - P_{S_\lambda } \big(y^{(\ell)}+ r_\ell\big),
\end{equation}
with
\begin{equation*}
S_\lambda = \menge{u = \big( u^{(k)}\big)_{1\leq k\leq K} \in [0,+\infty[^{K}}{\sum_{k=1}^K u^{(k)} = \lambda}.
\end{equation*}
\end{proposition}
\begin{proof}
Note that $u\in \RR^K \!\mapsto\! \lambda \max_{1\leq k \leq K} u^{(k)}$ is the support function of $S_\lambda$, defined as 
$(\forall u \in \RR^K)$
$\sigma_{S_\lambda}(u) = \sup_{v\in S_\lambda}  v^\top u$. Hence, for every ${y}^{(\ell)}\in \RR^K$, $\lambda h_\ell({y}^{(\ell)}) = \sigma_{S_\lambda}({y}^{(\ell)}+r_\ell)$ and  
\begin{equation*}
\prox_{\lambda h_\ell} \big(y^{(\ell)}\big) =  \prox_{\sigma_{S_\lambda}}(y^{(\ell)}+r_\ell)-r_\ell,
\end{equation*}
Since $\sigma_{S_\lambda}$ is the conjugate function of $\iota_{S_\lambda}$,
\eqref{eq:prox_max} is deduced by applying Moreau's decomposition formula \cite[Theorem~14.3(ii)]{Bauschke_H_2011_book_con_amo} and \eqref{e:projprox}.
\end{proof}

The iterations associated with Problem~\eqref{eq:regularized_SVM_hinge} are summarized in Algorithm~\ref{algo:SVM_reg}, where the sequence $({\rm x}^{[i]})_{i\in \NN}$ is guaranteed to converge to a solution to Problem~\eqref{eq:regularized_SVM_hinge}, provided that such a solution exists \citep{Vu_B_2011_j-acm_spl_adm,Condat_L_2012}. In Algorithm~\ref{algo:SVM_reg}, we use the notation
\begin{equation*}
T = [T_1^\top\;\dots\;T_L^\top]^\top,
\qquad
r = [r_1^\top\;\dots\;r_L^\top]^\top.
\end{equation*}

\begin{remark}
	Algorithm~\ref{algo:SVM_reg} allows us to solve Problem~\eqref{eq:regularized_SVM_hinge} together with its (Fenchel-Rockafellar) dual formulation
	\begin{equation}\label{eq:dual_SVM}
	\minimize{y \in \RR^{LK}}\; g^*(-T^\top y) - \sum_{\ell=1}^L r_\ell^\top y^{(\ell)} \quad\subto\quad y \in (S_\lambda)^L,
	\end{equation} 
	where $g^*$ is the convex conjugate of $g$. In the case when $g = (1/2)\|\cdot\|^2_2$, the primal and dual solutions are linked by ${\rm x} = -T^\top y$, and thus Problem~\eqref{eq:dual_SVM} reduces to the (Lagrangian) dual formulation of Problem~\eqref{eq:slack_SVM} used in standard SVMs \citep{Crammer_K_2001_j-mach-learn-res_algorithmic_imk}.
\end{remark}

\begin{algorithm}
	\caption{FBPD for solving Problem~\eqref{eq:regularized_SVM_hinge}}\label{algo:SVM_reg}
	{\small
		\vspace{1em}
		Initialization
		\[
		\!\!\!\!\!\!\!\!\!\!\!\!\!\!\!\!\!\!\!\!\!\!\!
		\left\lfloor
		\begin{aligned}
		&\textrm{choose ${{\rm x}^{[0]}} \in \RR^{(M+1)K} $}\\
		&\textrm{choose ${y^{[0]}} \in \RR^{LK} $}\\
		&\textrm{set $\tau>0$ and $\sigma>0$ such that $\tau\sigma \|T\|^2 \le 1$.}\\
		\end{aligned}
		\right.\\
		\]
		
		\noindent For\; i = 0, 1, \dots
		
		\[
		\!\!\!\!\!\!\!\!\!\!\!\!\!\!\!\!\!\!\!\!\!\!\!\!\!\!\!\!\!\!\!\!\!\!\!\!\!\!\!\!\!\!\!\!\!\!\!
		\left\lfloor
		\begin{aligned}
		&{\rm x}^{[i+1]} = \prox_{\tau g}\big({\rm x}^{[i]} -\tau \, T^\top y^{[i]}\big)\\
		&\widehat{y}^{[i+1]} = y^{[i]}+\sigma T\big(2{\rm x}^{[i+1]}-{\rm x}^{[i]}\big)\\
		&y^{[i+1]} = P_{(S_\lambda)^L}\left( \widehat{y}^{[i+1]} + \sigma \, r \right).
		\end{aligned}
		\right.
		\]
	}
\end{algorithm}

\subsection{Constrained formulation}
Problem~\eqref{eq:constrained_SVM_hinge} presents a more challenging computational issue, as the projection onto the hinge-loss constraint set cannot be evaluated in closed form, and it would require to solve a constrained quadratic problem at each iteration. In order to manage this constraint, we propose to introduce a vector of auxiliary variables $\zeta = \big(\zeta^{(\ell)}\big)_{1 \le \ell \le L}$ in the minimization process, so that Problem~\eqref{eq:constrained_SVM_hinge} can be equivalently rewritten as
\begin{equation}\label{eq:epigraphical_SVM}
\minimize{({\rm x},\zeta)\in \RR^{(M+1)K}\times \RR^L} \; g({\rm x}) \quad \subto \quad
\left\{
\begin{aligned} 
&\sum_{\ell=1}^L \zeta^{(\ell)} \leq \eta,\\
&(\forall \ell\in \{1,\ldots,L\}) && h_\ell(T_\ell \, {\rm x}) \le \zeta^{(\ell)}.
\end{aligned}
\right.
\end{equation}
Interestingly, our approach is conceptually similar to adding the slack variables in \eqref{eq:svm_constraint}, even though our reformulation specifically aims at simplifying the 
way of solving the problem. Indeed, a possible interpretation of Problem~\eqref{eq:epigraphical_SVM} is the following:
\begin{equation}\label{eq:prob_epi}
\minimize{({\rm x},\zeta)\in \RR^{(M+1)K}\times \RR^L} g({\rm x}) \quad\subto\quad 
\begin{cases}
(T{\rm x}, &\!\!\!\! \zeta) \in E,\\
     &\!\!\!\! \zeta\hphantom{)} \in V_\eta,
\end{cases}
\end{equation}
where $E$ denotes the collection of epigraphs of $h_1,\dots,h_L$
\begin{equation*}
E = \menge{(y,\zeta) \in \RR^{LK}\times\RR^L}{
(\forall \ell \in \{1,\dots,L\}) \quad (y^{(\ell)},\zeta^{(\ell)}) \in \epi h_\ell},
\end{equation*}
and $V_\eta$ denotes a closed half-space
\begin{equation*}
V_\eta = \menge{\zeta \in \RR^L}{\sum_{\ell=1}^L \zeta^{(\ell)} \le \eta}.
\end{equation*}
The iterations related to Problem~\eqref{eq:prob_epi} are listed in Algorithm~\ref{algo:SVM_epi}, where the sequence $({\rm x}^{[i]},\zeta^{[i]})_{i\in \NN}$ is guaranteed to converge to a solution to \eqref{eq:prob_epi}, provided that such a solution exists \citep{Vu_B_2011_j-acm_spl_adm,Condat_L_2012}. 

\begin{algorithm}
\caption{FBPD for solving Problem~\eqref{eq:constrained_SVM_hinge}}\label{algo:SVM_epi}
{\small
\vspace{1em}
Initialization
\[
\left\lfloor
\begin{aligned}
&\textrm{choose $\parens{{\rm x}^{[0]}, \zeta^{[0]}} \in \RR^{(M+1)K} \times  \RR^L$}\\
&\textrm{choose $\parens{y^{[0]}, \xi^{[0]}} \in \RR^{L(K-1)} \times \RR^L$}\\
&\textrm{set $\tau>0$ and $\sigma>0$ such that $\tau\sigma \max\{\|T\|^2,1\} \le 1$}.\\
\end{aligned}
\right.\\
\]

\noindent For\; i = 0, 1, \dots

\[
\!\!\!\!\!\!\!\!\!\!\!\!\!\!\!\!\!\!\!\!\!\!\!\!\!\!\!\!\!\!\!\!\!\!\!\!\!\!\!\!\!\!\!\!\!\!\!
\left\lfloor
\begin{aligned}
&{\rm x}^{[i+1]} = \prox_{\tau g}\big({\rm x}^{[i]} -\tau \, T^\top y^{[i]}\big)\\
&\zeta^{[i+1]} = P_{V_\eta}\big(\zeta^{[i]}-\tau\, \xi^{[i]}\big)\\
& \widehat{y}^{[i]} = y^{[i]}+\sigma T\big(2{\rm x}^{[i+1]}-{\rm x}^{[i]}\big)\\
& \widehat{\xi}^{[i]} = \xi^{[i]}+\sigma \big(2\zeta^{[i+1]}-\zeta^{[i]}\big)\\
&\big(\widetilde{y}^{[i]},\widetilde{\xi}^{[i]}\big) =P_E\big(\widehat{y}^{[i]}/\sigma,\widehat{\xi}^{[i]}/\sigma\big)\\
&y^{[i+1]} = \widehat{y}^{[i]} - \sigma \widetilde{y}^{[i]}\\
&\xi^{[i+1]} = \widehat{\xi}^{[i]} - \sigma \widetilde{\xi}^{[i]}.
\end{aligned}
\right.
\]
}
\end{algorithm}

The advantage of our approach lies in the fact that the projections onto $E$ and $V_\eta$ employed in Algorithm~\ref{algo:SVM_epi} have closed form expressions. Indeed, the projection onto $V_\eta$ is straightforward \cite[Section 6.2.3]{Parikh_N_2014_j-fto_proximal_a}, while the projection onto $E$ can be block-decomposed as 
\begin{equation}\label{eq:epi_decomp}
P_E(y,\zeta) = \Big( P_{\epi h_\ell}(y^{(\ell)},\zeta^{(\ell)}) \Big)_{1\le\ell\le L}
\end{equation}
where a closed-form expression of $P_{\epi h_\ell}$ with $\ell \in \{1,\ldots,L\}$ is given in Proposition~\ref{ex:approx}. The proof of this new result follows the same line as the proof by \cite[Proposition~5]{Chierchia_G_2014_SIVP_epigraphical_ppt}, where we derived the epigraphical projection associated to the $\ell_\infty$-norm.

The decomposition in \eqref{eq:epi_decomp} yields two potential benefits. Firstly, the projection $P_{\epi h_\ell}$ is computed onto the lower-dimensional convex subset $\epi h_\ell$ of $\RR^{K}\times\RR$, whose dimensionality is only fixed by the number $K$ of classes. Secondly, these projections can be computed in parallel, since they are defined over disjoint blocks whose number is given by the cardinality $L$ of the training set (we refer to \citet{Gaetano2012} for an example of parallel implementation on GP-GPUs).

\begin{proposition}\label{ex:approx}
For every $\ell\in\{1,\ldots,L\}$, let $h_\ell$ be the function defined in \eqref{eq:fun_h} and, for every {\small$({ y}^{(\ell)},\zeta^{(\ell)}) \in \RR^{K}\times\RR$}, let {\small$\big(\nu^{(\ell,k)}\big)_{1\le k \le K}$} be the sequence {$\big(y^{(\ell,k)}+r^{(k)}_\ell\big)_{1\le k \le K}$} sorted in ascending order,\footnote{Note that the expensive sorting operation can be avoided by using a \emph{heap} data structure \citep{Cormen1990_algorithms}, which keeps a partially-sorted sequence such that the first element is the largest. This approach was used, e.g., by van den Berg et al.~\cite[Algorithm 2]{VanDenBerg_E_2008_j-siam-sci-comp_pro_pfb} for implementing the projection onto the $\bell_1$-ball.} and set {\small$\nu^{(\ell,0)} = -\infty$} and {\small$\nu^{(\ell,K+1)} = \pinf$}. Then, $P_{\epi h_\ell}({y}^{(\ell)},\zeta^{(\ell)}) = ({p}^{(\ell)},\theta^{(\ell)})$ with
\begin{equation*}
p^{(\ell)} = \Big[\, \min\{y^{(\ell,k)}, \theta^{(\ell)}-r^{(k)}_\ell\} \,\Big]_{1\le k\le K}
\end{equation*}
and
\begin{equation}\label{e:projscalmaxfuncter}
\theta^{(\ell)} = 
\frac{1}{K-\overline{k}^{(\ell)}+2} \left(\zeta^{(\ell)}+\sum_{k=\overline{k}^{(\ell)}}^{K}\nu^{(\ell,k)}\right),
\end{equation}
where $\overline{k}^{(\ell)}$ is the unique integer in {\small$\{1,\ldots,K+1\}$}
such that
\begin{equation}\label{e:projscalmaxfuncterbis}
\nu^{(\ell,\overline{k}^{(\ell)}-1)} < \theta^{(\ell)} \le \nu^{(\ell,\overline{k}^{(\ell)})},
\end{equation}
with the convention $\sum_{k=K+1}^{K} \cdot = 0$.
\end{proposition}
\begin{proof}
For every $(y^{(\ell)},\zeta^{(\ell)}) \in \RR^{K}\times\RR$, $P_{\epi h_\ell}(y^{(\ell)},\zeta^{(\ell)})$ denotes the unique solution to
\begin{equation*}
\min_{(p^{(\ell)},\theta^{(\ell)}) \in \epi h_\ell}
\|p^{(\ell)}-y^{(\ell)}\|^2+(\theta^{(\ell)}-\zeta^{(\ell)})^2,
\end{equation*}
which is equivalent to find the minimizer of
\begin{equation}\label{e:projDimax}
\min_{\theta^{(\ell)} \in \RR} \; \Big\{(\theta^{(\ell)}-\zeta^{(\ell)})^2  +\!\!\!\!
\min_{\substack{\;\;\;{ p}^{(\ell,1)}\le \theta^{(\ell)} - r_\ell^{(1)}\\\vdots\\
{ p}^{(\ell,K)}\le  \theta^{(\ell)} -r_\ell^{(K)}}}
\|{ p}^{(\ell)}-{ y}^{(\ell)}\|^2\Big\}.
\end{equation}
For every $\theta^{(\ell)}\in \RR$, the inner minimization is achieved when $p^{(\ell,k)} \!=\! \min\{y^{(\ell,k)}, \theta^{(\ell)}-r^{(k)}_\ell\}$ for each $k\in\{1,\dots,K\}$, reducing \eqref{e:projDimax} to
\begin{equation*}
\underset{\theta^{(\ell)} \in \RR}{\operatorname{min}}\;\Big\{(\theta^{(\ell)}-\zeta^{(\ell)})^2 +
\sum_{k=1}^{K} (\max\{{y}^{(\ell,k)}+r^{(k)}_\ell-\theta^{(\ell)},0\})^2\Big\},
\end{equation*}
which achieves its minimum when $\theta^{(\ell)} = \prox_{\varphi_\ell}(\zeta^{(\ell)})$, with
\begin{equation}\label{e:varphiinf}
(\forall v \in \RR)\qquad
\varphi_\ell(v) = \frac12\sum_{k=1}^{K} 
(\max\{{y}^{(\ell,k)}+r^{(k)}_\ell-v,0\})^2.
\end{equation}
The closed-form expression of this proximity operator, as well as the projection onto $\epi h_\ell$, are derived in the following. 
In order to prove \eqref{e:projscalmaxfuncter}, we need to compute the proximity operator of $\varphi_\ell$ defined in \eqref{e:varphiinf}. Such a function belongs to $\Gamma_0(\RR)$ since, for each $k\!\in\! \{1,\ldots,K\}$, $\max\{(\nu^{(\ell,k)} - \cdot),0\}$ is finite convex and $(\cdot)^2$ is finite convex and increasing on $\RP$. In addition, $\varphi_\ell$ is differentiable and such that, for every $v\in \RR$ and $k \in \{1,\ldots,K+1\}$,
\begin{equation*}
\nu^{(\ell,k-1)} < v \le \nu^{(\ell,k)} \quad\Rightarrow \quad
 \varphi_\ell(v) =
\frac12 \sum_{j=k}^{K} \big(v-\nu^{(\ell,j)} \big)^2.
\end{equation*}
Therefore, there exists a $\bar{k}^{(\ell)} \in \{1,\ldots,K+1\}$ such that $\nu^{(\ell,\bar{k}^{(\ell)}-1)} < \theta^{(\ell)} \le \nu^{(\ell,\bar{k}^{(\ell)})}$, which yields \eqref{e:projscalmaxfuncterbis}. Moreover, by the definition of proximity operator, $\theta^{(\ell)} = \prox_{\varphi_\ell}(\zeta^{(\ell)})$ is uniquely defined by $\zeta^{(\ell)}-\theta^{(\ell)} = \varphi'_\ell(\theta^{(\ell)})$, yielding
\begin{equation*}
\zeta^{(\ell)}-\theta^{(\ell)}  = \sum_{k=\bar{k}^{(\ell)}}^{K} (\theta^{(\ell)} - \nu^{(\ell,k)}),
\end{equation*}
which is equivalent to \eqref{e:projscalmaxfuncter}.
The uniqueness of $\bar{k}^{(\ell)}$
follows from that of $\prox_{\varphi_\ell}(\zeta^{(\ell)})$. 
\end{proof}


\section{Numerical results}\label{s:nr}
In this section, we numerically evaluate the performance of sparse multiclass SVM w.r.t.\ the three following databases.

\begin{itemize}
\item \textbf{Leukemia database.}
The first experiment concerns the classification of microarray data. The considered database contains $72$ samples of $N=M=7129$ gene expression levels (so that $\phi(u) = u$)  measured from patients having $K=3$ types of leukemia disease \citep{Golub1999}. The database is usually organized in $L=38$ training samples and $34$ test samples.\footnote{Data available at \textsl{www.broadinstitute.org/cancer/software/genepattern/datasets}} In our experiments, we used blocks of $5$ genes for the mixed-norm regularization.

\item \textbf{MNIST dataset.}
The second experiment concerns the classification of handwritten digits. More precisely, we consider the MNIST database \citep{LeCun_1998_procIEEE}, which contains a number of $28 \times 28$ grayscale images ($N=784$) displaying digits from $0$ to $9$ ($K=10$). The database is organized in $60000$ training images and $10000$ test images.\footnote{Data available at \textsl{http://yann.lecun.com/exdb/mnist}} In our experiments, we defined the mapping $\phi$ by resorting to the scattering convolution network \citep{Bruna_J_2013_j-ieee-pami_scattering_networks} with $\overline{m} = 2$ wavelet layers scaled up to $2^J=4$, which transforms an input image of size $28 \times 28$ in $81$ images of size $14\times 14$ (thus $M=15876$). For the regularization, we used the $\bell_{1,\infty}$-norm by dividing each vector $({\rm x}^{(k)})_{1\le k\le K}$ in $14^2$ blocks of size $81$. Moreover, in order to evaluate the performance, we trained a classifier on $25$ different training subsets of size $L\in\{3K,5K,10K\}$, we computed the classification errors by evaluating the $25$ trained classifiers on the whole test set, and we averaged the resulting errors. 

\item \textbf{News20 database.}
The third experiment concerns the classification of text documents into a fixed number of predefined categories. More precisely, we consider the News20 database \citep{Lang95}, which contains a number of documents partitioned across $K=20$ different newsgroups. The database is organized in $11314$ training documents and $7532$ test documents.\footnote{Data available at \textsl{www.cad.zju.edu.cn/home/dengcai/Data/TextData.html}}
In our experiments, we defined the mapping $\phi$ by resorting to the \emph{term frequency -- inverse document frequency} transformation \citep{Joachims1998}, yielding $M=26214$. For the regularization, we used $\bell_{1,2}$-norm in the same way as \citet{Blondel_2013_j-mach-learn_block_coord_SMV}. Moreover, in order to evaluate the performance, we trained a classifier on $10$ different training subsets of size $L\in\{5K,10K,50K\}$, we computed the classification errors by evaluating the $10$ trained classifiers on the whole test set, and  we averaged the resulting errors.
\end{itemize}

\subsection{Assessment of classification accuracy}
In this section, we evaluate the classification errors obtained with the sparse multiclass SVM formulated in Problems~\eqref{eq:regularized_SVM_hinge}-\eqref{eq:constrained_SVM_hinge}. Our objective here is to show that the \emph{exact} hinge loss allows us to achieve better performance than its approximated smooth versions, especially with a few training data. Hence, we compare the proposed method with the following approaches:
\begin{itemize}
\item the multiclass SVM proposed by \citet{Blondel_2013_j-mach-learn_block_coord_SMV}
\begin{equation}\label{eq:square}
\operatorname*{minimize}_{{\rm x} \in \RR^{(M+1)K}}\;  \, g({\rm x}) + \\
 \lambda\sum_{\ell=1}^L \sum_{k\ne z_\ell} \left(\max\big\{0, \mu_\ell+ \varphi(u_\ell)^\top ( {\rm x}^{(k)} - {\rm x}^{(z_\ell)})\big\}\right)^2,
\end{equation}

\item the multinomial logistic regression (e.g., see   \citet{Krishnapuram_2005_j-pami-multi_logit})
\begin{equation}\label{eq:logit}
\operatorname*{minimize}_{{\rm x} \in \RR^{(M+1)K}}\;  g({\rm x}) + \\
 \lambda\sum_{\ell=1}^L \log\Bigg(1 + \sum_{k\ne z_\ell} \exp\Big\{\mu_\ell+ \varphi(u_\ell)^\top ( {\rm x}^{(k)} - {\rm x}^{(z_\ell)})\Big\}\Bigg).
\end{equation}

\item the binary SVM by \citet{Laporte2014_j-ieee-tnnls_sparseSVM} based on the ``one-vs-all" strategy, which aims, for every $k \in \{1,\dots,K\}$, to
\begin{equation}\label{eq:one_vs_all}
\operatorname*{minimize}_{{\rm x}^{(k)} \in \RR^{(M+1)}}\;  \, g({\rm x}) + \\
\lambda\sum_{\ell=1}^L \left(\max\big\{0, \mu_\ell+ \widetilde{z}_\ell \; \varphi(u_\ell)^\top {\rm x}^{(k)}\big\}\right)^2,
\end{equation}
with $\widetilde{z}_\ell$ being equal to $1$ if $z_\ell = k$, and $-1$ otherwise. Note that \eqref{eq:one_vs_all} may be seen as a special case of \eqref{eq:square}.

\end{itemize}

In the following, we refer to Problems~\eqref{eq:regularized_SVM_hinge}-\eqref{eq:constrained_SVM_hinge} as \emph{hinge}, and to Problems~\eqref{eq:square}-\eqref{eq:one_vs_all}, respectively, as \emph{square}, \emph{logit}, and \emph{one-vs-all}. Since the parameters $\lambda$ and $\eta$ need to be estimated (e.g., through cross validation), it is important to evaluate the impact of their choice on the performance, although it is out of the scope of this paper to devise an optimal strategy to set this bound. To compare the above methods for different choices of these parameters, we set $\lambda = \alpha^{-1}$ or $\eta = \alpha L$, by varying $\alpha$ inside a fixed set of predefined values. We also follow the usual convention of setting $\mu_\ell\equiv 1$.

\begin{itemize}

\item \textbf{Leukemia database}. Table~\ref{tab:leukemia} reports the classification errors, as well as the number of non-zero coefficients in vectors $({\rm x}^{(k)})_{1\le k\le 3}$, obtained with \emph{hinge}, \emph{square}, \emph{logit}, and \emph{one-vs-all} using various regularization terms. For each method, we set $\alpha$ to the value yielding the best accuracy (by using a simple trial-and-error strategy). The results indicate that sparse regularization allows us to effectively select a small set of important features for each prediction vector $({\rm x}^{(k)})_{1\le k\le 3}$, with better results than the quadratic regularization. In addition, the classification errors show that \emph{hinge} is often more accurate than \emph{square}.

\begin{table*}
	\centering
	\caption{Comparisons on the leukemia database.}
	\label{tab:leukemia}
	{\tiny
		\begin{tabular}[b]{ccccccccc}
			\toprule
			$g({\rm x})$ & \multicolumn{2}{c}{\textsc{hinge}} & \multicolumn{2}{c}{\textsc{square}} & \multicolumn{2}{c}{\textsc{logit}} & \multicolumn{2}{c}{\textsc{one-vs-all}}\\
			\cmidrule(rl){2-3}\cmidrule(rl){4-5}\cmidrule(rl){6-7}\cmidrule(rl){8-9}
			& errors & non-zero coeff. & errors & non-zero coeff. & errors & non-zero coeff. & errors & non-zero coeff.\\
			\midrule
			$\bell_2$          & $\mathbf{1/34}$ & $7129+7129+7129$ &         $2/34$  & $7129+7129+7129$ & $\mathbf{1/34}$ & $7129+7129+7129$ & $2/34$ & $7129+7129+7129$ \\
			$\bell_1$          & $\mathbf{2/34}$ &       $13+03+10$ &         $3/34$  &          $8+3+8$ &         $3/34$  &       $18+05+14$ & $3/34$ & $19+8+15$ \\
			$\bell_{1,2}$      & $\mathbf{0/34}$ &       $95+5+75$  &         $1/34$  &       $55+05+45$ & $\mathbf{0/34}$ &       $50+05+35$ & $1/34$ & $70+10+50$\\
			$\bell_{1,\infty}$ & $\mathbf{0/34}$ &       $50+5+45$  & $\mathbf{0/34}$ &       $35+05+35$ & $\mathbf{0/34}$ &       $50+05+40$ & $\mathbf{0/34}$ & $45+5+45$\\
			\bottomrule
		\end{tabular}
	}
\end{table*}

\item \textbf{MNIST database}. Figures~\ref{fig:mnist:errors_L3}, \ref{fig:mnist:errors_L5} and \ref{fig:mnist:errors_L10} report the classification errors as a function of the regularization hyperparameter. These results were obtained with the $\bell_{1,\infty}$-norm regularization, as it was the one leading to the best results in all our experiments on this database. The classification errors indicate that the \emph{hinge} approach is slightly more accurate than the other ones. On the other side, Figures~\ref{fig:mnist:sparse_L3}, \ref{fig:mnist:sparse_L5} and \ref{fig:mnist:sparse_L10} report the percentage of zero coefficients in vectors $({\rm x}^{(k)})_{1\le k\le K}$ as a function of $\alpha$. The plots show that the \emph{hinge} approach yields solutions slightly more sparse than the other ones.

\item \textbf{News20 database}. Figures~\ref{fig:news20:errors_L3}, \ref{fig:news20:errors_L5} and \ref{fig:news20:errors_L10} report the classification errors (as a function of the regularization hyperparameter) obtained by using the $\bell_{1,2}$-norm regularization. The classification errors indicate that the \emph{hinge} approach is slightly more accurate than the \emph{square} approach. The plots also show that the results obtained with the hinge approach are more robust w.r.t.\ the choice of the regularization parameter. On the other side, Figures~\ref{fig:news20:sparse_L3}, \ref{fig:news20:sparse_L5} and \ref{fig:news20:sparse_L10} report the percentage of zero coefficients in vectors $({\rm x}^{(k)})_{1\le k\le K}$ as a function of $\alpha$. The plots show that the \emph{hinge} approach yields solutions as sparse as the \emph{square} approach.
\end{itemize}

\begin{figure}
	\centering
	\subfloat[$L/K=3$ (errors vs $\alpha$)]{ \includegraphics[width=0.45\textwidth]{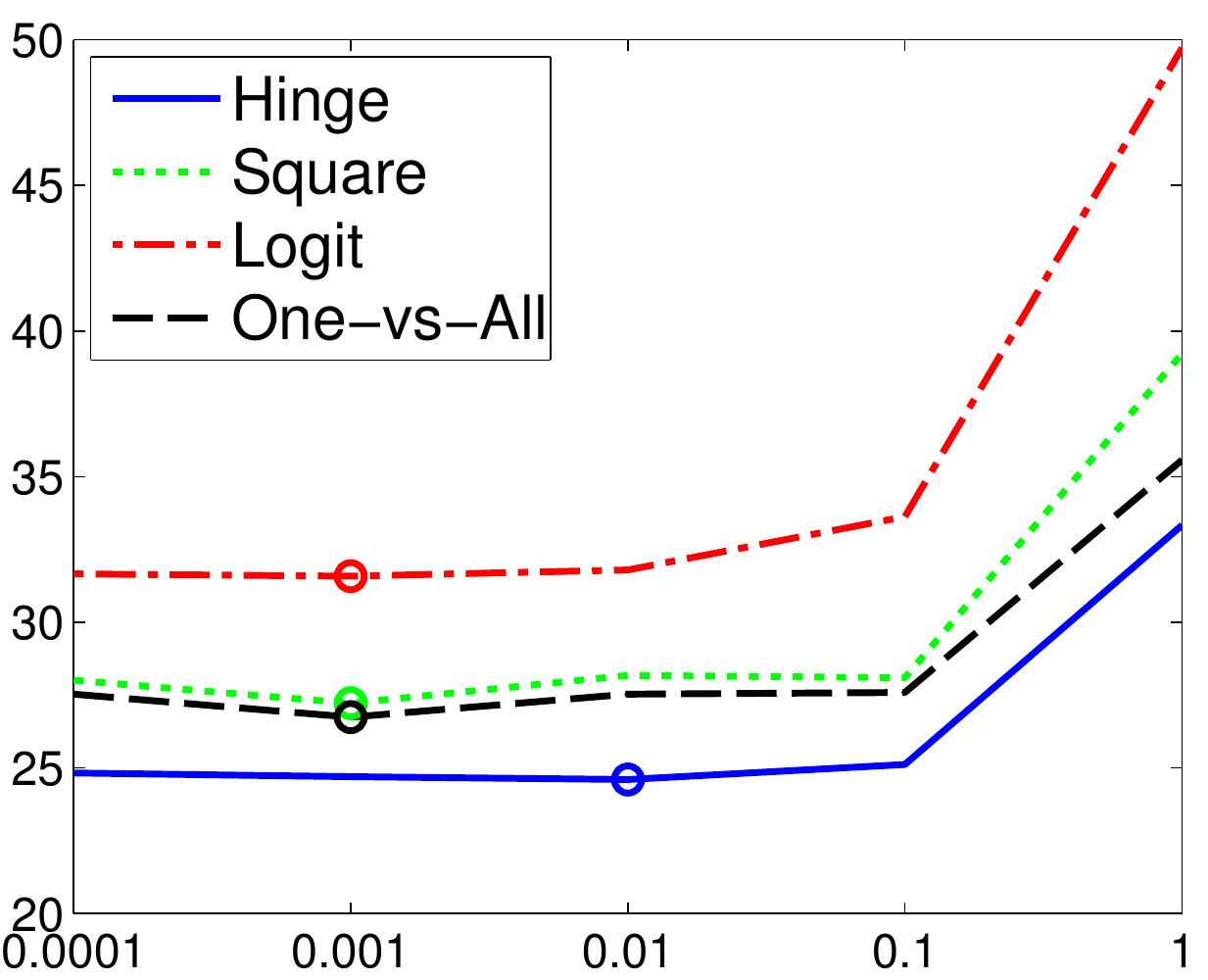}\label{fig:mnist:errors_L3}}
	\hfill
	\subfloat[$L/K=3$ (sparsity vs $\alpha$)]{ \includegraphics[width=0.45\textwidth]{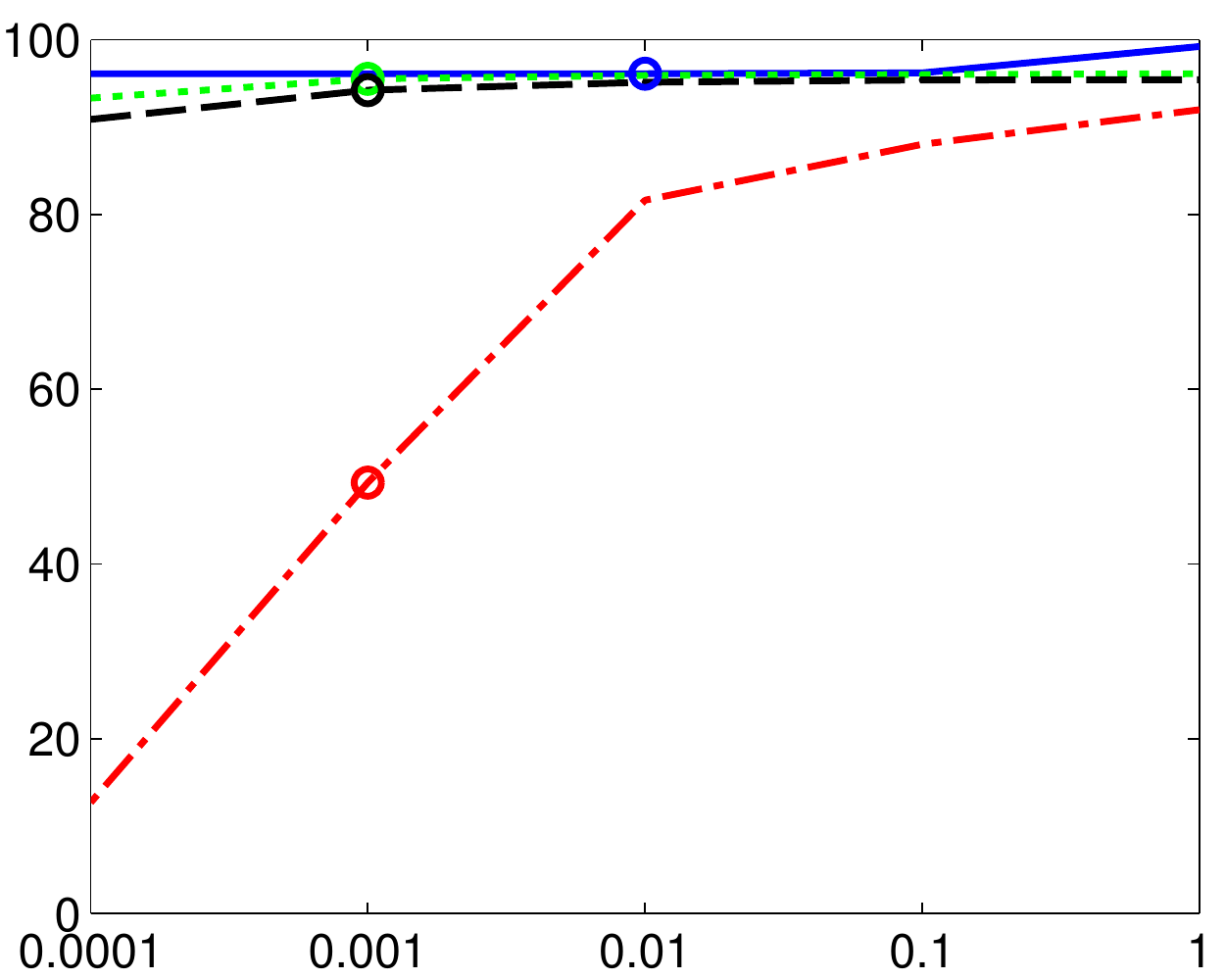}\label{fig:mnist:sparse_L3}}
	\hfill
	\subfloat[$L/K=5$ (errors vs $\alpha$)]{ \includegraphics[width=0.45\textwidth]{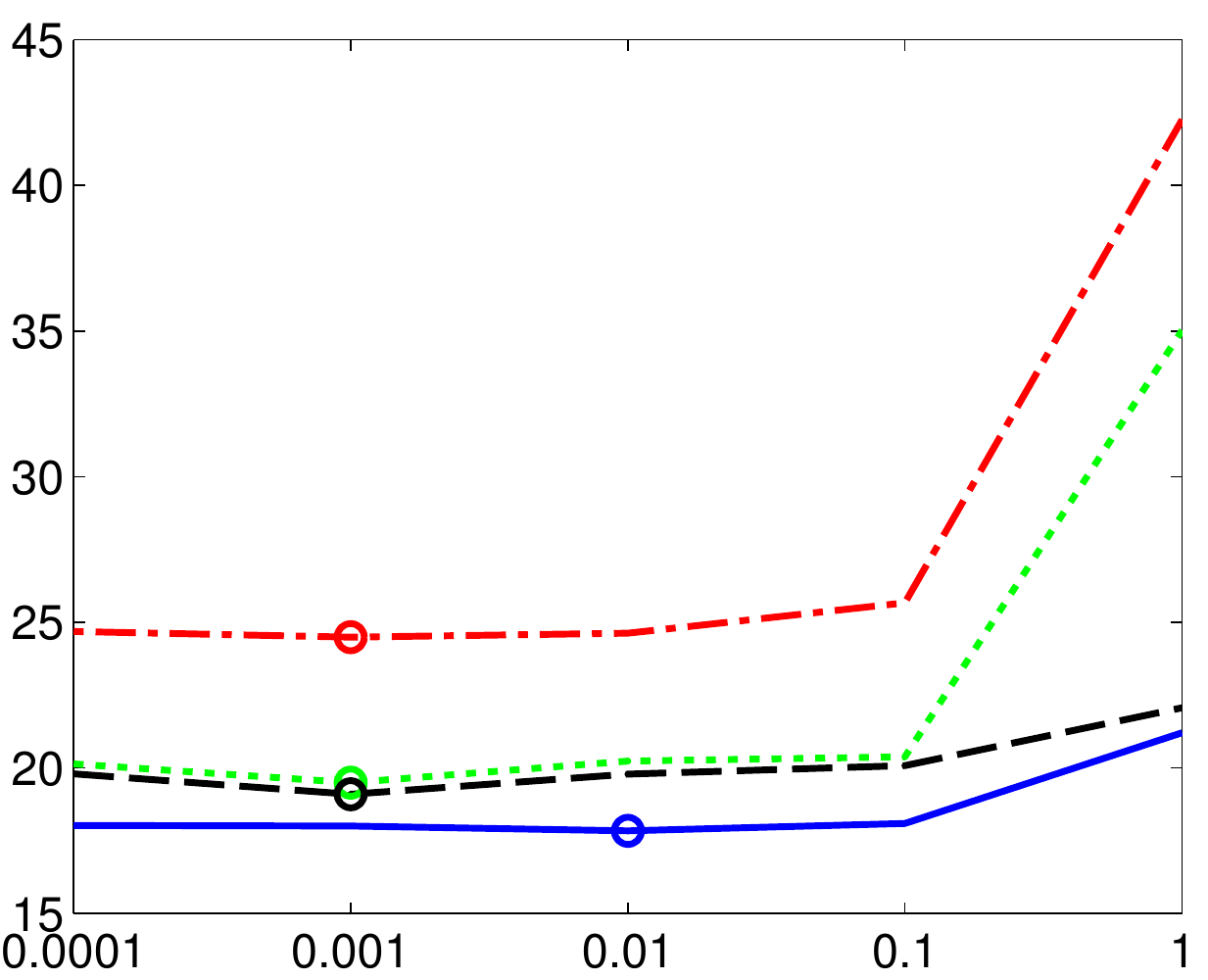}\label{fig:mnist:errors_L5}}
	\hfill
	\subfloat[$L/K=5$ (sparsity vs $\alpha$)]{ \includegraphics[width=0.45\textwidth]{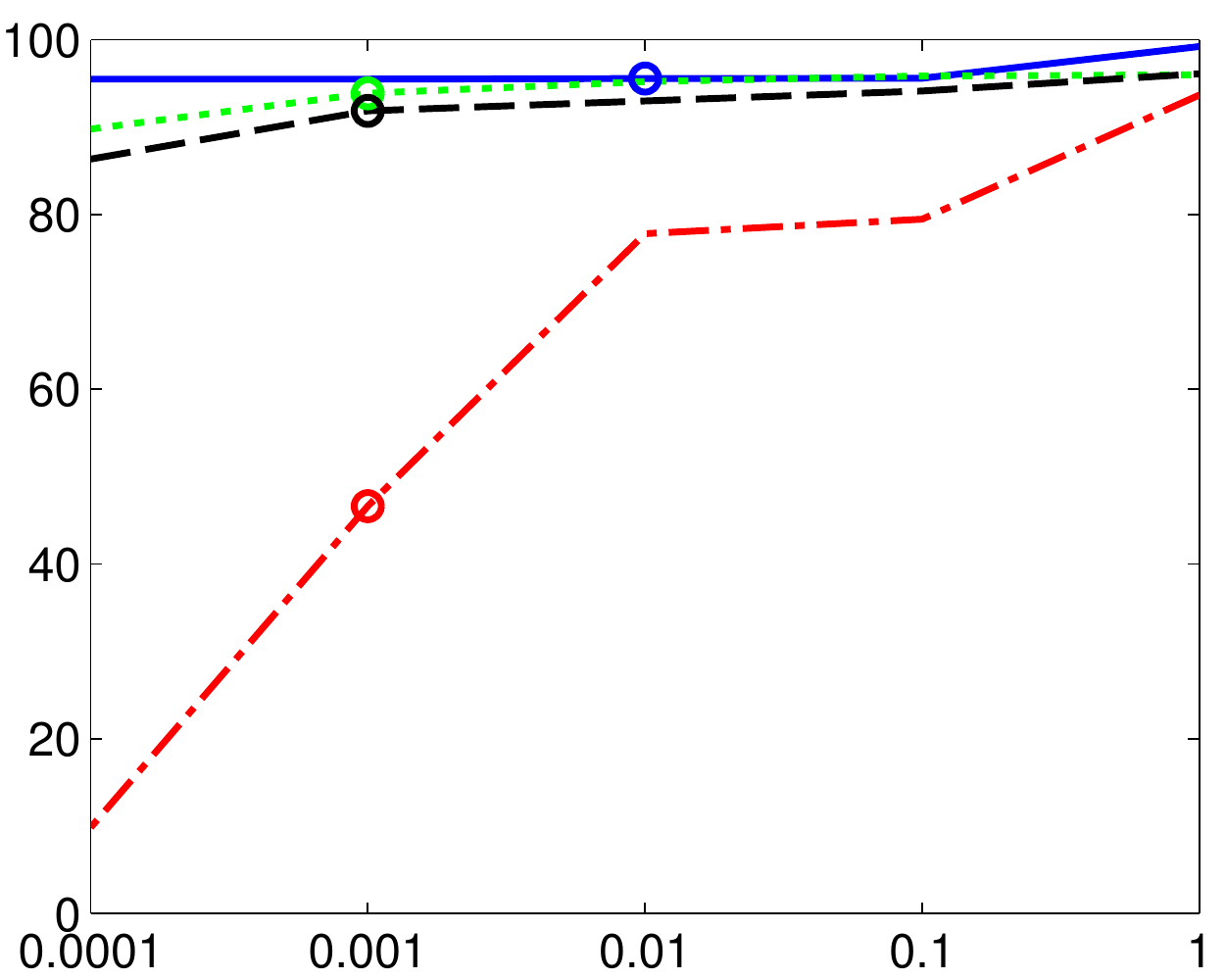}\label{fig:mnist:sparse_L5}}
	\hfill
	\subfloat[$L/K=10$ (errors vs $\alpha$)]{ \includegraphics[width=0.45\textwidth]{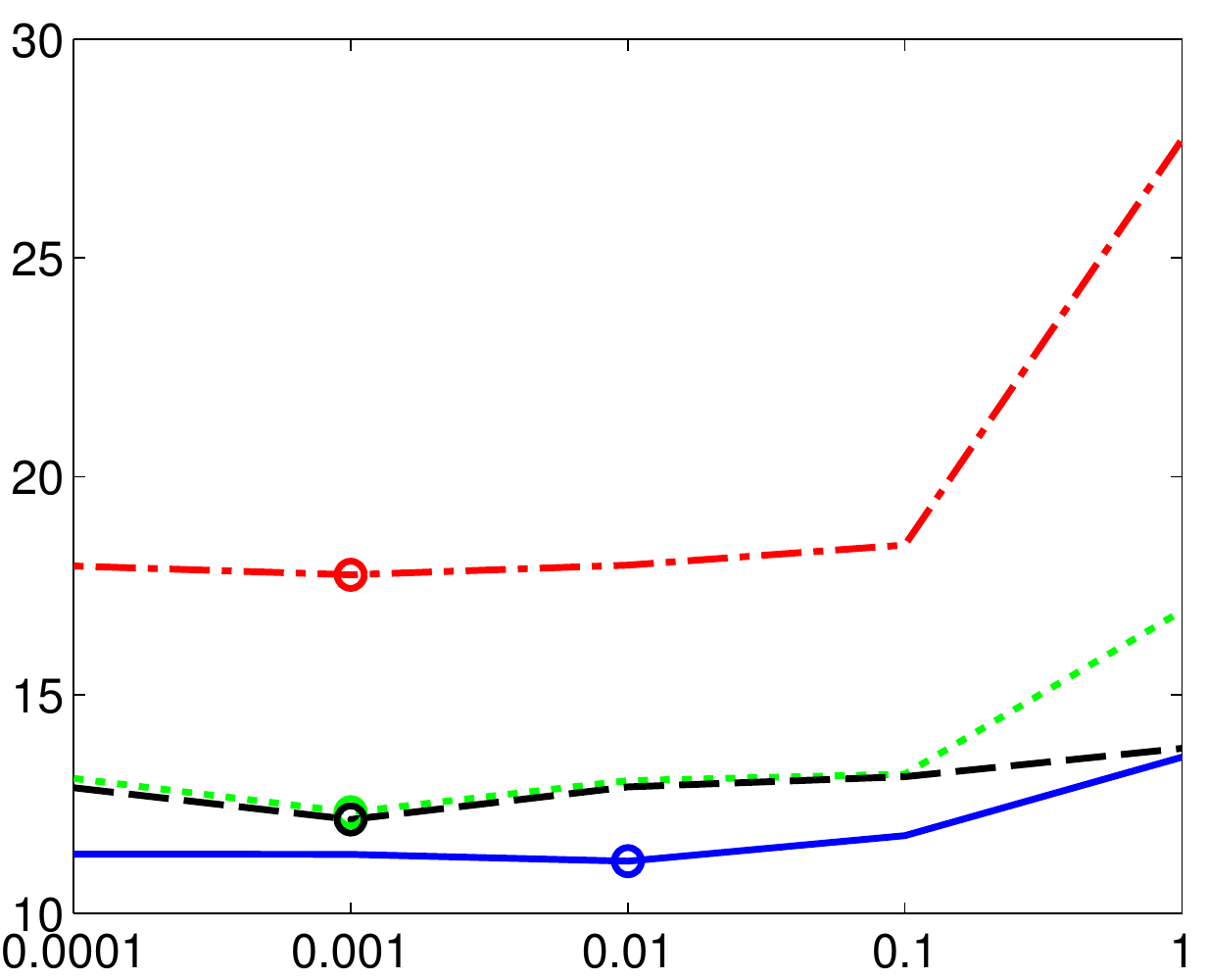}\label{fig:mnist:errors_L10}}
	\hfill
	\subfloat[$L/K=10$ (sparsity vs $\alpha$)]{ \includegraphics[width=0.45\textwidth]{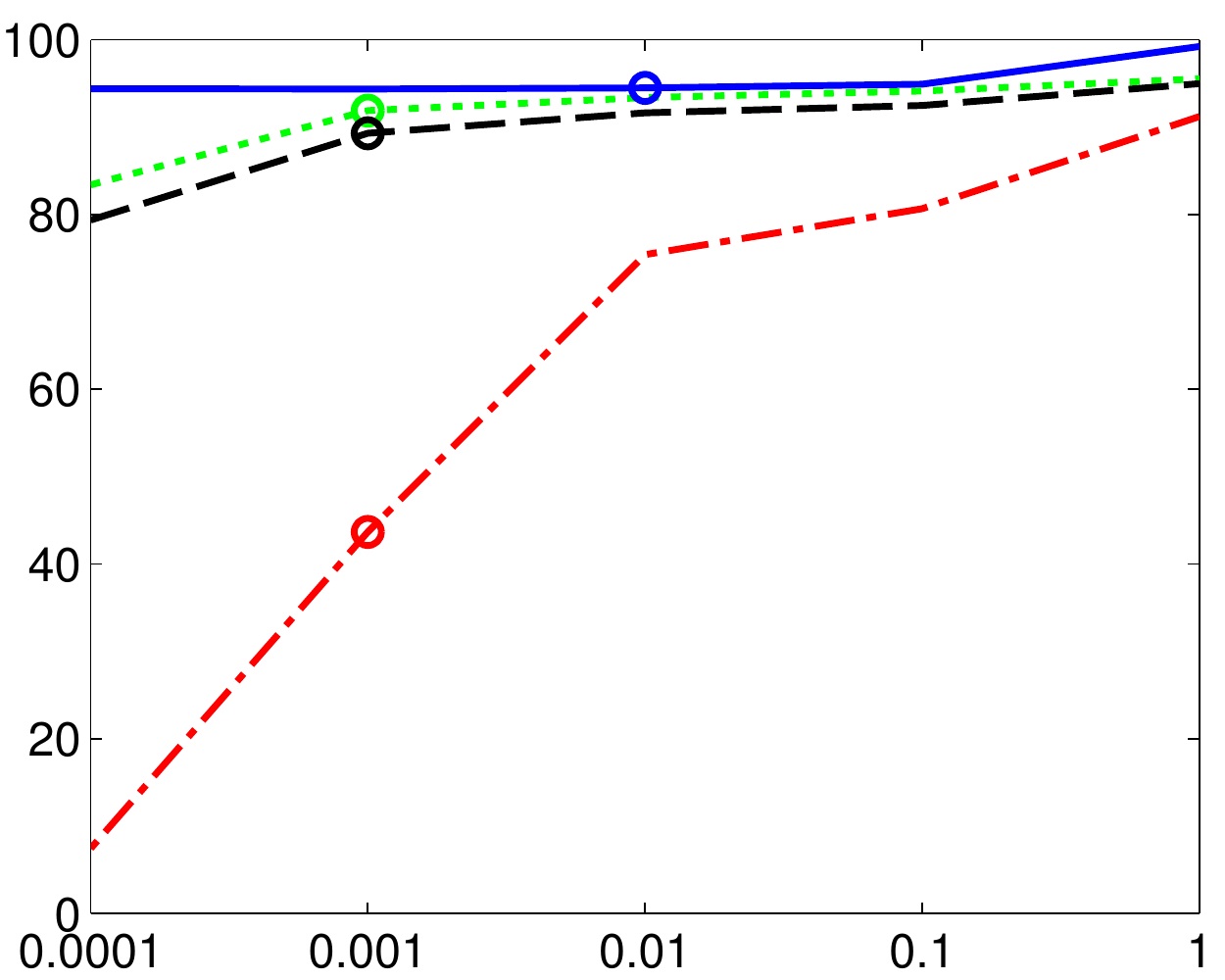}\label{fig:mnist:sparse_L10}}
	
	\caption{Results on MNIST database with the $\bell_{1,\infty}$-regularization for $L\in\{3K,5K,10K\}$. Left column: classification errors as a function of $\alpha$. Right column: percentage of zero coefficients in vectors $({\rm x}^{(k)})_{1\le k\le K}$ as a function of $\alpha$. The circles mark the values yielding the best accuracy.}
	\label{fig:mnist}
\end{figure}
		
\begin{figure*}
	\centering
	\subfloat[$L/K=5$ (errors vs $\alpha$)]{ \includegraphics[width=0.45\textwidth]{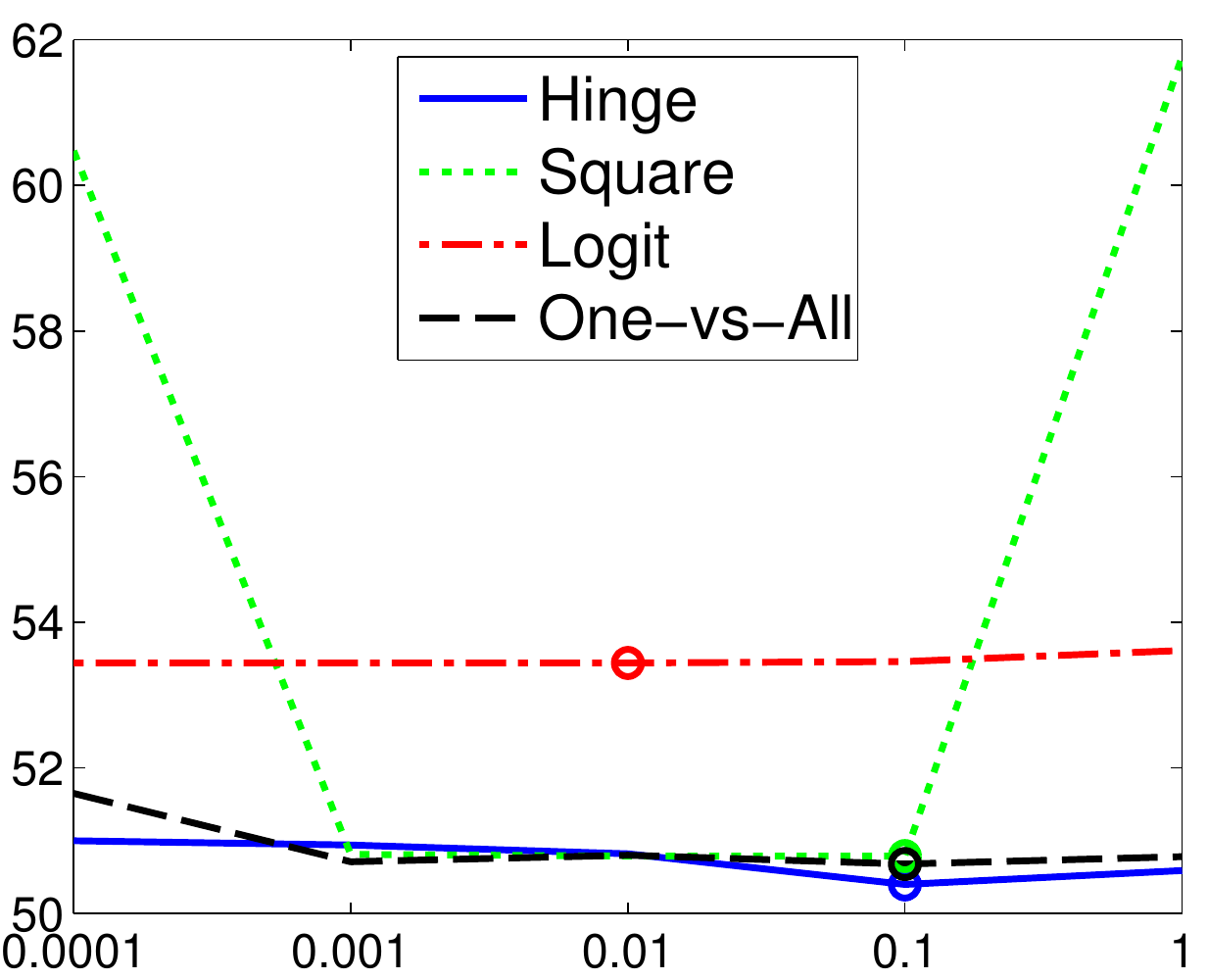}\label{fig:news20:errors_L3}}
	\hfill
	\subfloat[$L/K=5$ (sparsity vs $\alpha$)]{ \includegraphics[width=0.45\textwidth]{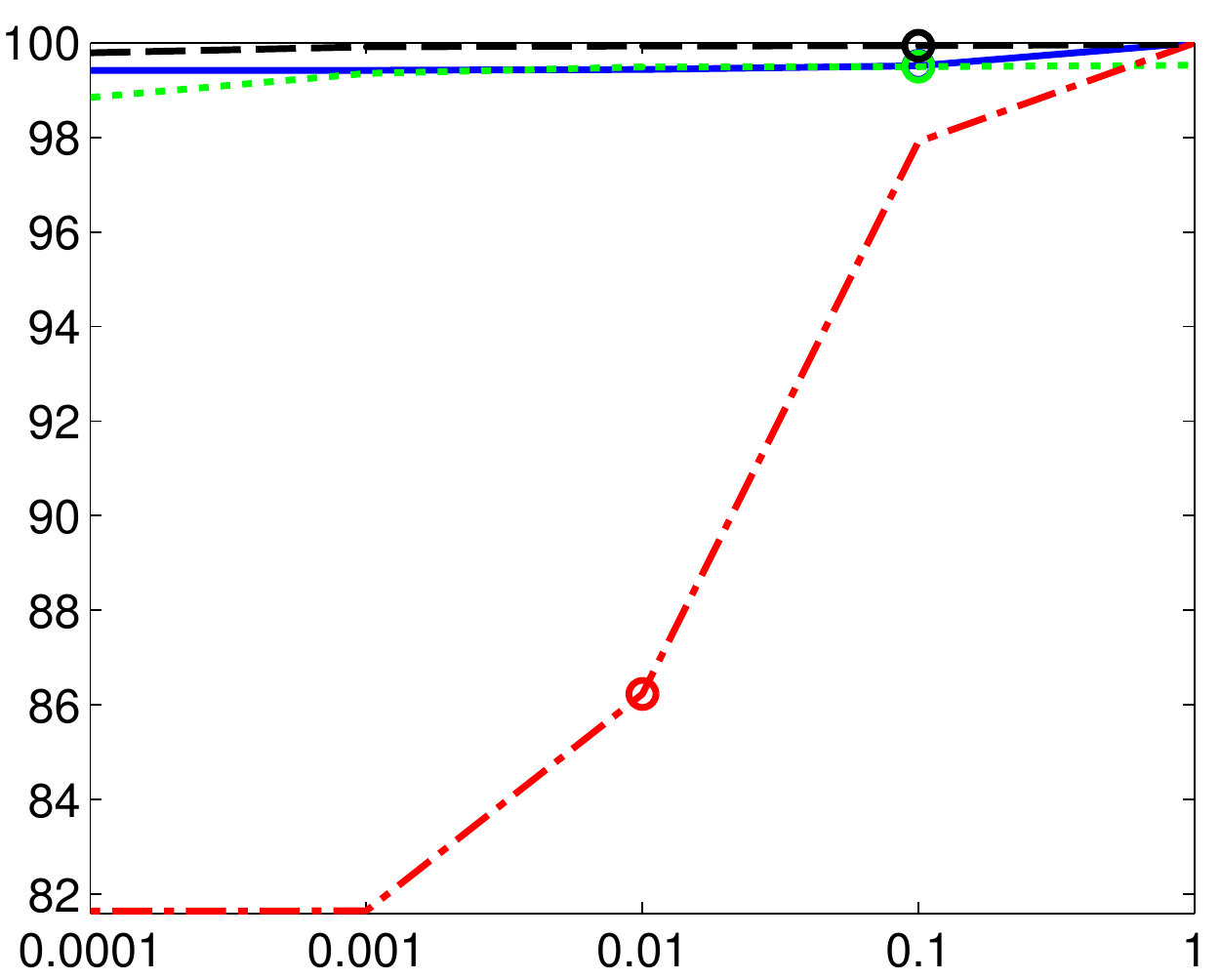}\label{fig:news20:sparse_L3}}
	\hfill
	\subfloat[$L/K=10$ (errors vs $\alpha$)]{ \includegraphics[width=0.45\textwidth]{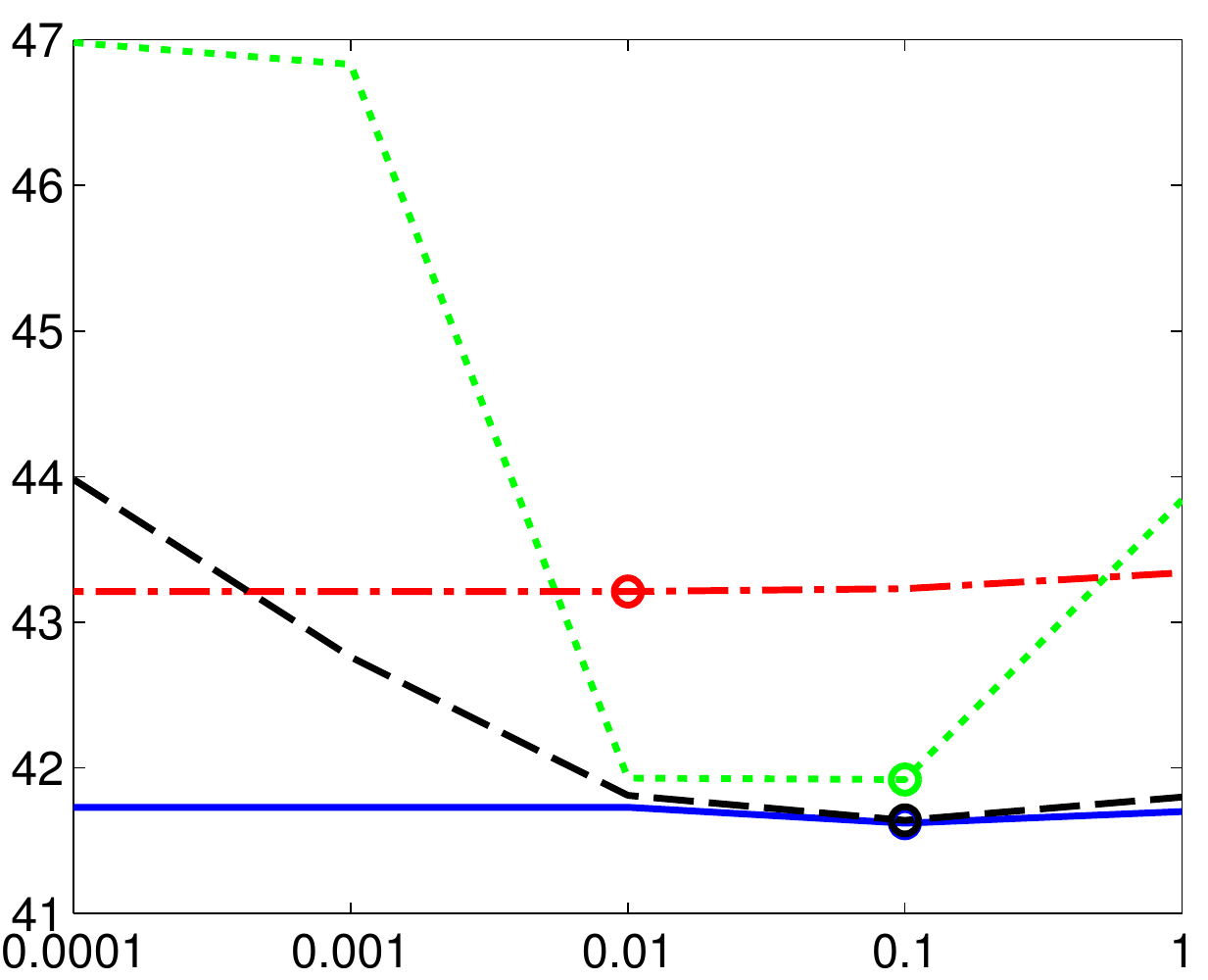}\label{fig:news20:errors_L5}}
	\hfill
	\subfloat[$L/K=10$ (sparsity vs $\alpha$)]{ \includegraphics[width=0.45\textwidth]{alpha_sparse_news20_L10}\label{fig:news20:sparse_L5}}
	\hfill
	\subfloat[$L/K=50$ (errors vs $\alpha$)]{ \includegraphics[width=0.45\textwidth]{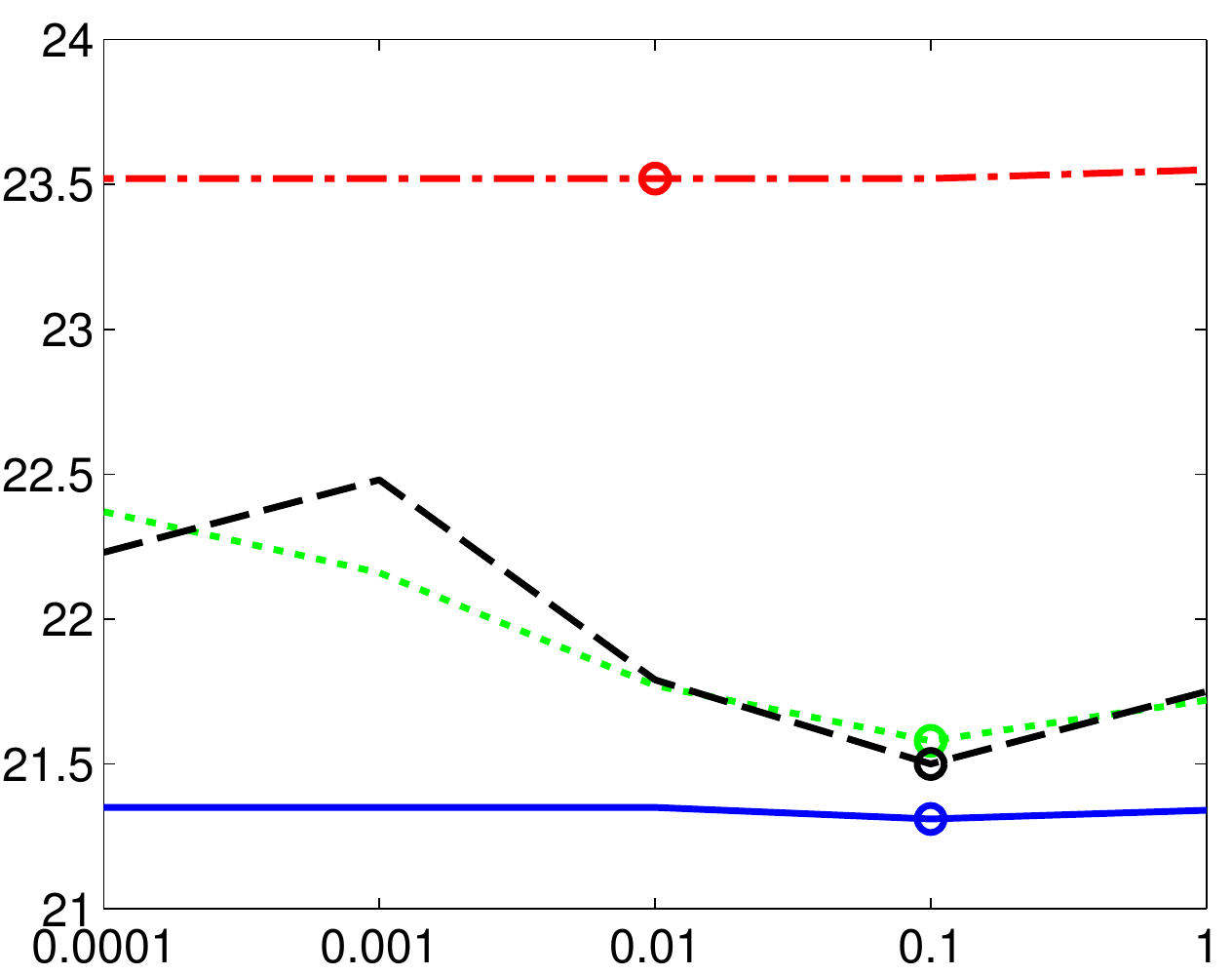}\label{fig:news20:errors_L10}}
	\hfill
	\subfloat[$L/K=50$ (sparsity vs $\alpha$)]{ \includegraphics[width=0.45\textwidth]{alpha_sparse_news20_L10}\label{fig:news20:sparse_L10}}
	\caption{Results on News20 database with the $\bell_{1,2}$-regularization for $L\in\{5K,10K,50K\}$. Left column: classification errors as a function of $\alpha$. Right column: percentage of zero coefficients in vectors $({\rm x}^{(k)})_{1\le k\le K}$ as a function of $\alpha$. The circles mark the values yielding the best accuracy.}
	\label{fig:news20}
\end{figure*}

\pagebreak
\subsection{Assessment of execution times}
In this section, we compare the execution times of Algorithms~\ref{algo:SVM_reg}~and~\ref{algo:SVM_epi} with\footnote{The codes were implemented in MATLAB and executed on a Intel CPU at 3.33 GHz and 24 GB of RAM.}\\[-1.5em]
\begin{itemize}

\item a FISTA implementation of Problem~\eqref{eq:square},\\[-1.5em]

\item a forward-backward implementation of Problem~\eqref{eq:logit},\\[-1.5em]

\item a FBPD implementation of Problem~\eqref{eq:epigraphical_SVM} reformulated with  linear constraints\\[-0.5em]
\begin{equation*}
\minimize{({\rm x},\zeta)\in \RR^{(M+1)K}\times \RR^{LK}}\; g({\rm x}) \quad\subto\quad
\begin{cases}
&\displaystyle\sum_{\ell=1}^L\sum_{k=1}^K \zeta^{(\ell,k)} \le K \eta,\\
&(\forall \ell\in \{1,...,L\}) \quad \zeta^{(\ell,1)} = \dots = \zeta^{(\ell,K)},\\
&(\forall \ell\in \{1,...,L\}) \quad 
\zeta^{(\ell,1)}\ge 0,\dots,\zeta^{(\ell,K)} \ge 0,\\
&(\forall \ell\in \{1,...,L\}) \quad 
T_\ell \, {\rm x} + r_\ell - (\zeta^{(\ell,k)})_{1 \le k \le K} \le 0. 
\end{cases}
\end{equation*}
This approach is conceptually similar to the linear programming methods proposed by \citet{Wang2007_stat_L1MSVM} and \citet{Zhang2008_variable-selection}  for $\bell_{1}$- or  $\bell_{1,+\infty}$-regularized SVMs.

\end{itemize}
Figures~\ref{fig:times:a1}, \ref{fig:times:a2} and \ref{fig:times:a3} show the execution times (averaged among $10$ training sets) obtained by the above algorithms for various values of $\lambda$ and $\eta$ on the MNIST database with $L\in\{3K,5K,10K\}$. In this experiment, the execution times refer to a stopping criterion of $10^{-5}$ on the relative error between two consecutive iterates. 
Conversely, Figures~\ref{fig:conv:a}, \ref{fig:conv:b} and \ref{fig:conv:c} show the relative distance to $\|{\rm x}^{[i]} - {\rm x}^{[\infty]}\| / \|{\rm x}^{[\infty]}\|$ (as a function of time) for the values of $\lambda$ and $\eta$ yielding the best accuracy (as reported in Figure~\ref{fig:mnist}), where ${\rm x}^{[\infty]}$ denotes the solution computed with a stopping criterion of $10^{-5}$. These results demonstrate that the proposed algorithms are faster than the approaches based on linear constraints and logistic regression, while being comparable in terms of execution times to approaches based on the square hinge loss. In addition, Algorithm~\ref{algo:SVM_epi} turns out to converge faster than Algorithm~\ref{algo:SVM_reg}. This can be explained by the higher computational cost of the projection onto the standard simplex.
%
%

\begin{figure*}
\centering
\subfloat[$L/K=3$ (average time vs $\alpha$)]{ \includegraphics[width=0.44\textwidth]{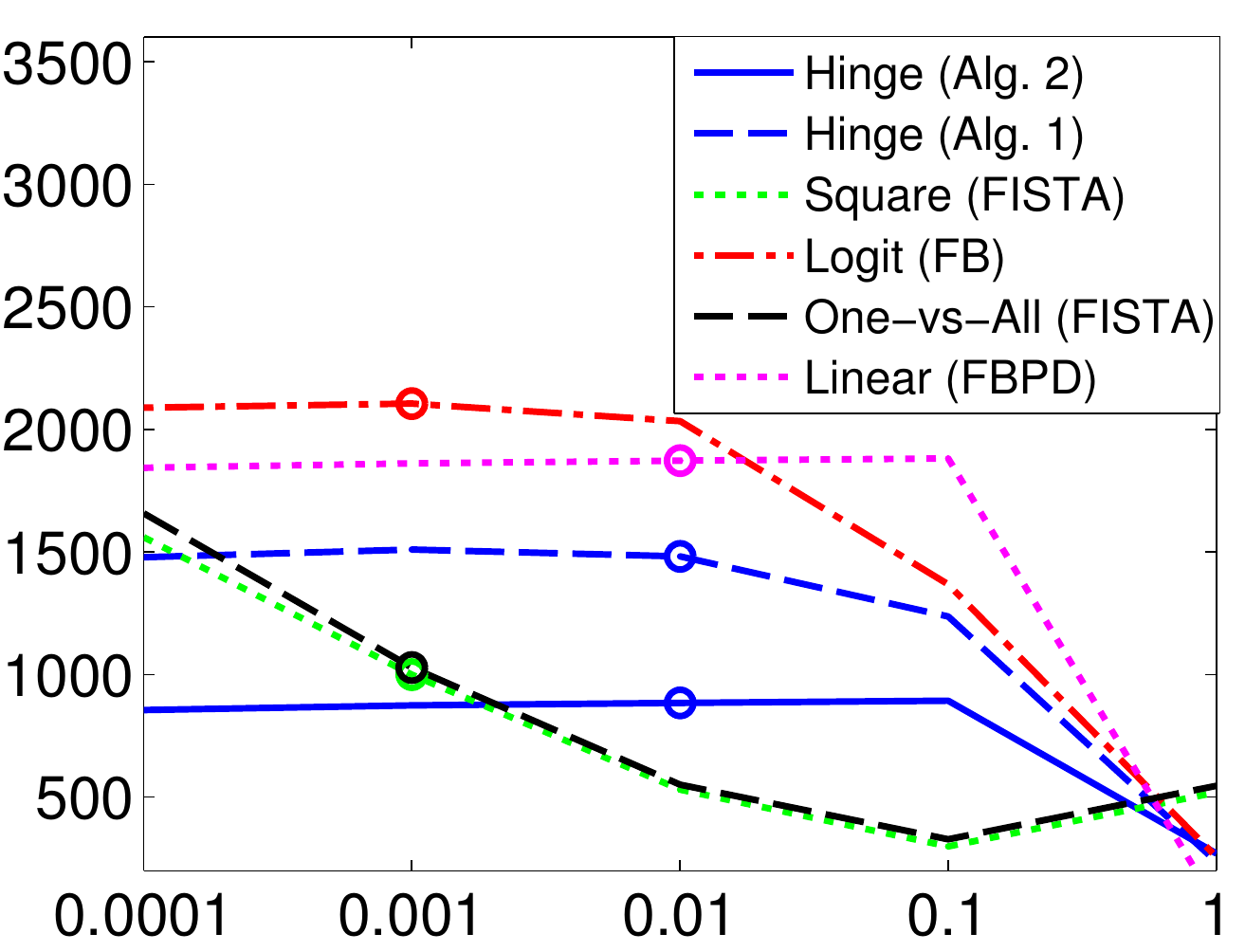}\label{fig:times:a1}}
\hfill
\subfloat[$L/K=3$ (distance vs time)]{ \includegraphics[width=0.5\textwidth]{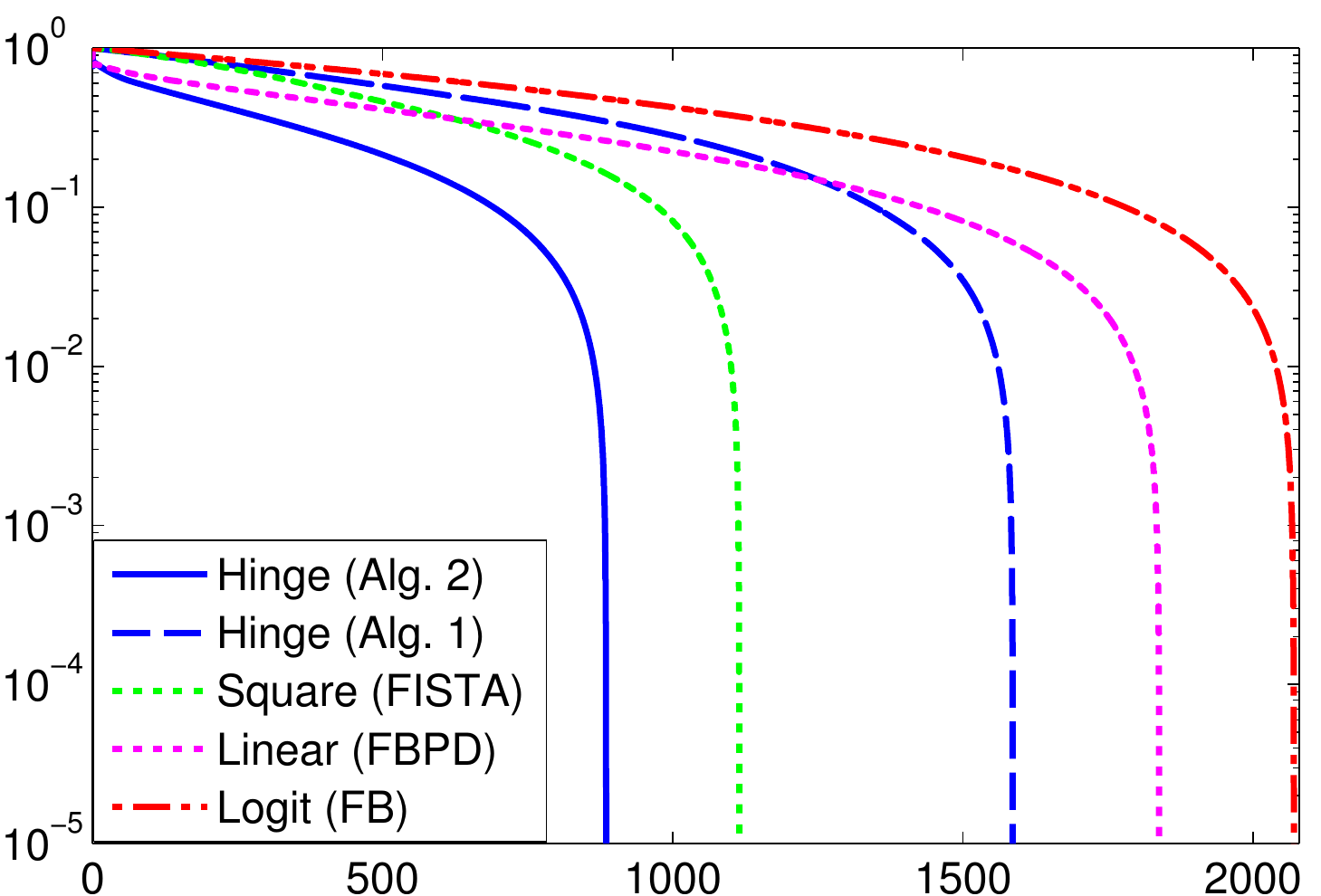}\label{fig:conv:a}}
\hfill
\subfloat[$L/K=5$ (average time vs $\alpha$)]{ \includegraphics[width=0.44\textwidth]{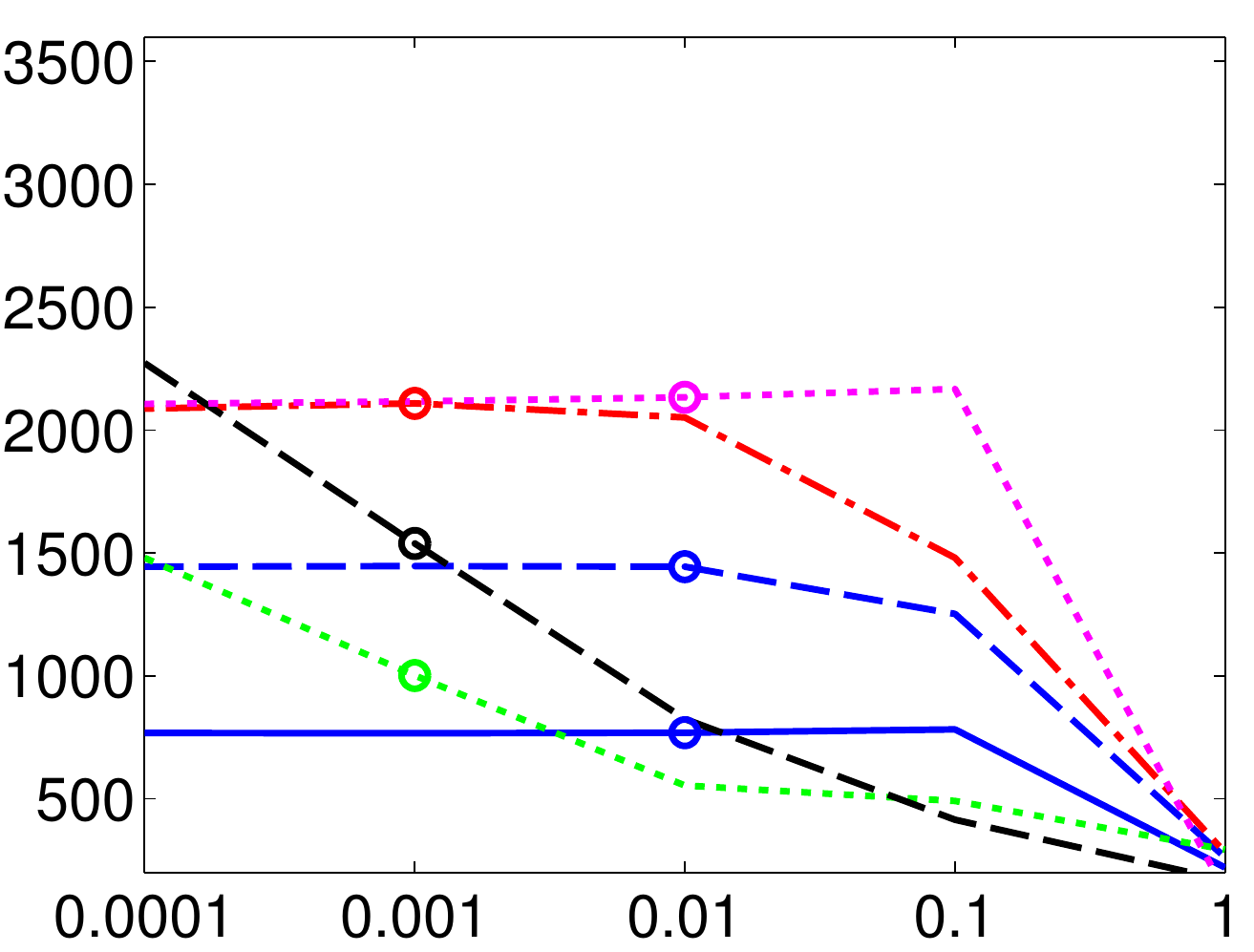}\label{fig:times:a2}}
\hfill
\subfloat[$L/K=5$ (distance vs time)]{ \includegraphics[width=0.5\textwidth]{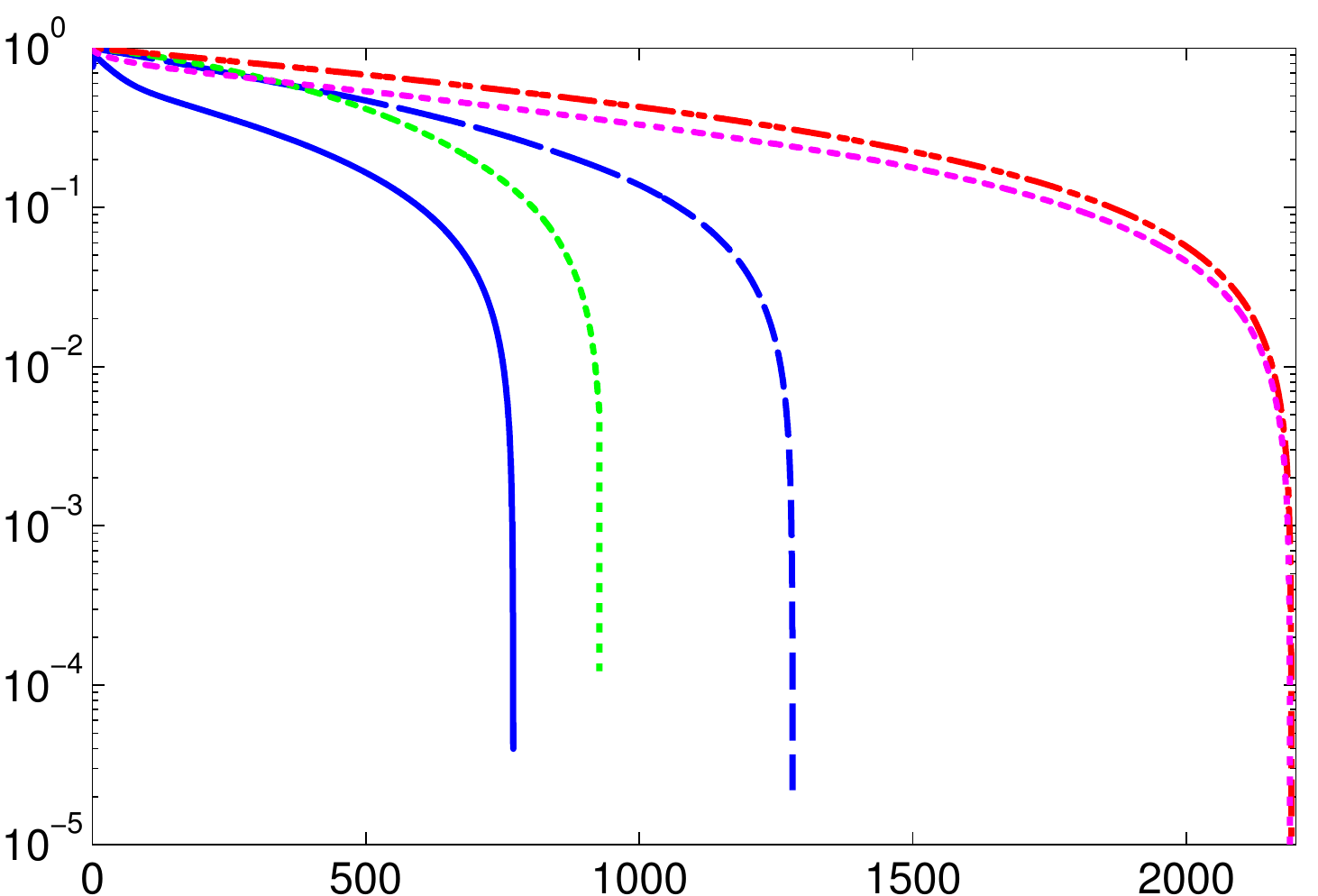}\label{fig:conv:b}}
\hfill
\subfloat[$L/K=10$ (average time vs $\alpha$)]{ \includegraphics[width=0.44\textwidth]{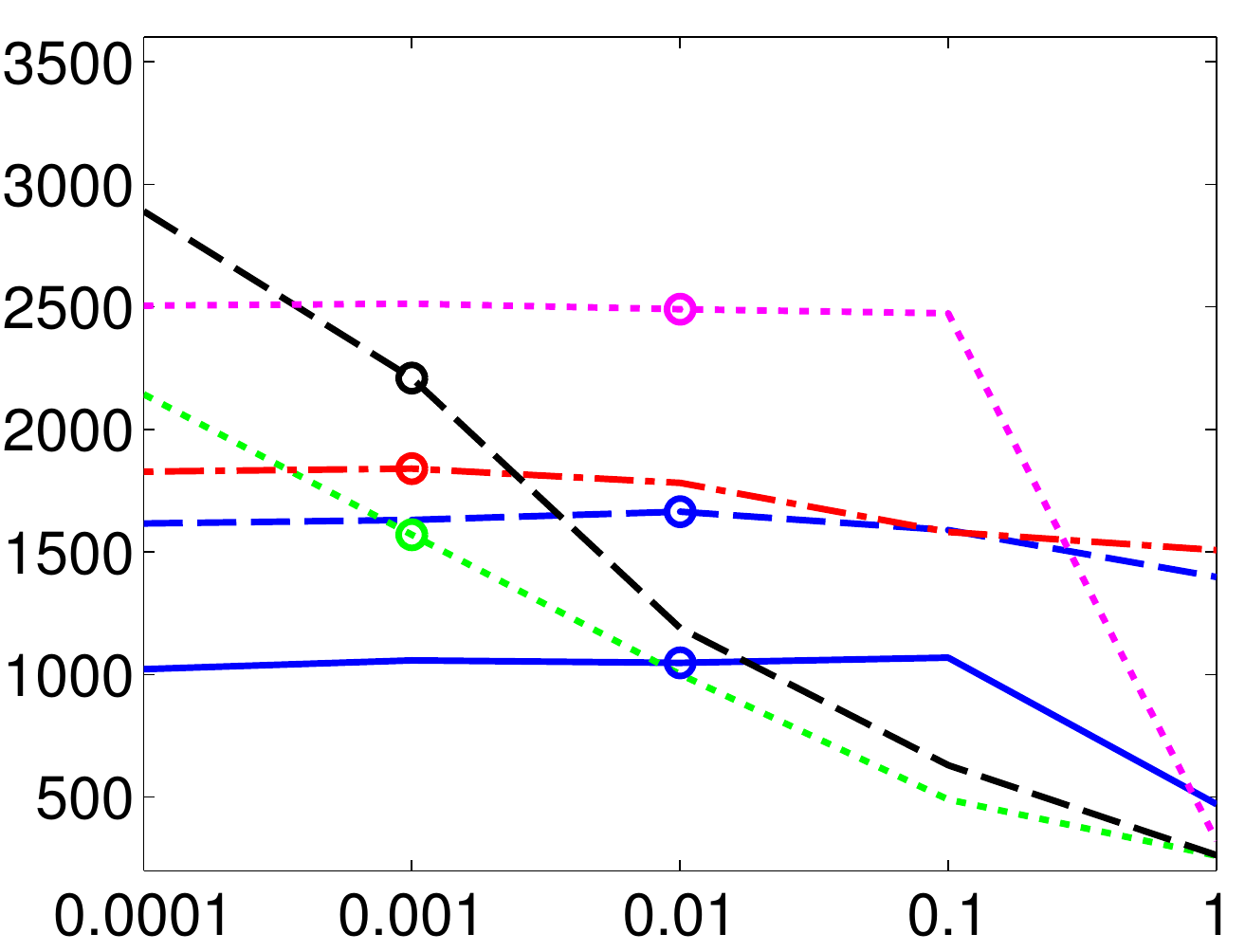}\label{fig:times:a3}}
\hfill
\subfloat[$L/K=10$ (distance vs time)]{ \includegraphics[width=0.5\textwidth]{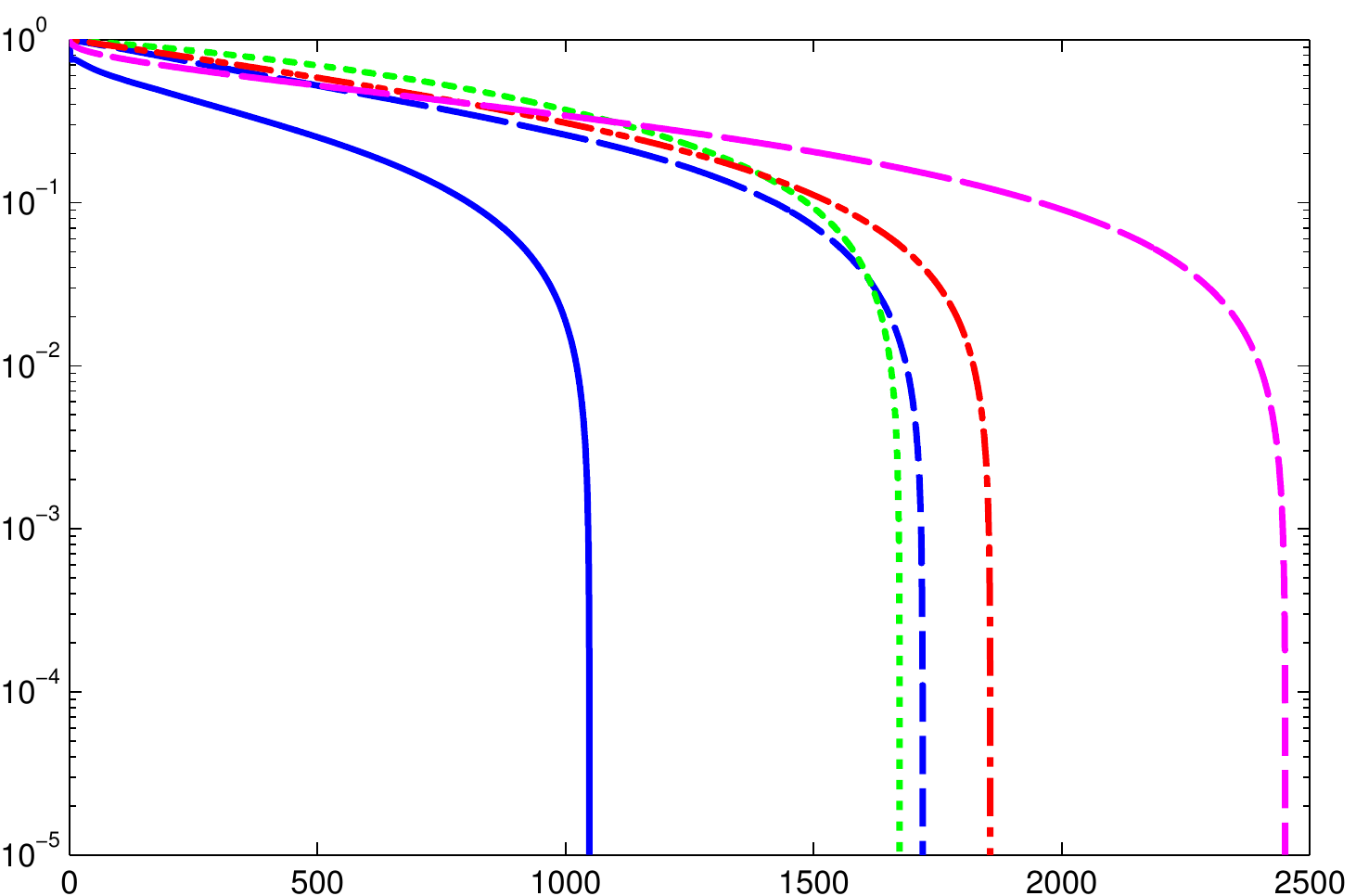}\label{fig:conv:c}}

\caption{Results on MNIST database with the $\bell_{1,\infty}$-regularization for $L\in\{3K,5K,10K\}$. Left column: execution time as a function of $\alpha$, where the circles mark the values yielding the best accuracy (as reported in Figure~\ref{fig:mnist}). Right column: distance to ${\rm x}^{[\infty]}$ (as a function of time) obtained with the values of $\alpha$ marked by a circle in the left column (note that the one-vs-all approach, being defined by multiple optimization problems, does not allow us to determine the iterate $x^{[i]}$ at each iteration, hence the associated plot cannot be traced).}
\label{fig:times}
\end{figure*}

\subsection{Quadratic regularization}
Although our emphasis is on sparse learning, we propose to complete our analysis by evaluating the efficiency of the proposed algorithms in the case when $g$ is a quadratic regularization function. To this end, we compare the execution times of Algorithms~\ref{algo:SVM_reg}~and~\ref{algo:SVM_epi} with the \emph{SVM-struct} algorithm proposed by \citet{Joachims_T_2009_j-mach-learn_cutting_pts}, which provides a numerical approach for solving Problem~\eqref{eq:slack_SVM} through a cutting-plane technique. Figure~\ref{fig:times_L2} reports the execution times (averaged on $10$ training sets) obtained by the above methods on the MNIST database with $L\in\{3K, 5K, 10K,50K,100K, 500K\}$ and different values of $\alpha$. In this experiment, we set the stopping criterion to $10^{-3}$ in all methods, and the regularization parameter of \emph{SVM-struct} to $L/\alpha$. The results show that the proposed algorithms are competitive with state-of-the-art solutions in scenarios
with a limited number of training data. The same cannot be claimed for larger databases, as \emph{SVM-struct} scales particularly well w.r.t.\ the number $M$ of features and the size $L$ of the training set. Note however that, when $L/K = 500$, the number of significant features for the SVM classifier designed with a quadratic regularization is equal to $M-546 = 158214$ (by setting a threshold to $10^{-5}$), while a sparse approach using an $\ell_{1,\infty}$-norm regularization yields only $42795$ nonzero features.

%
%

\begin{figure*}
	\centering
	\subfloat[$L/K=3$ (average time vs $\alpha$)]{ \includegraphics[width=0.45\textwidth]{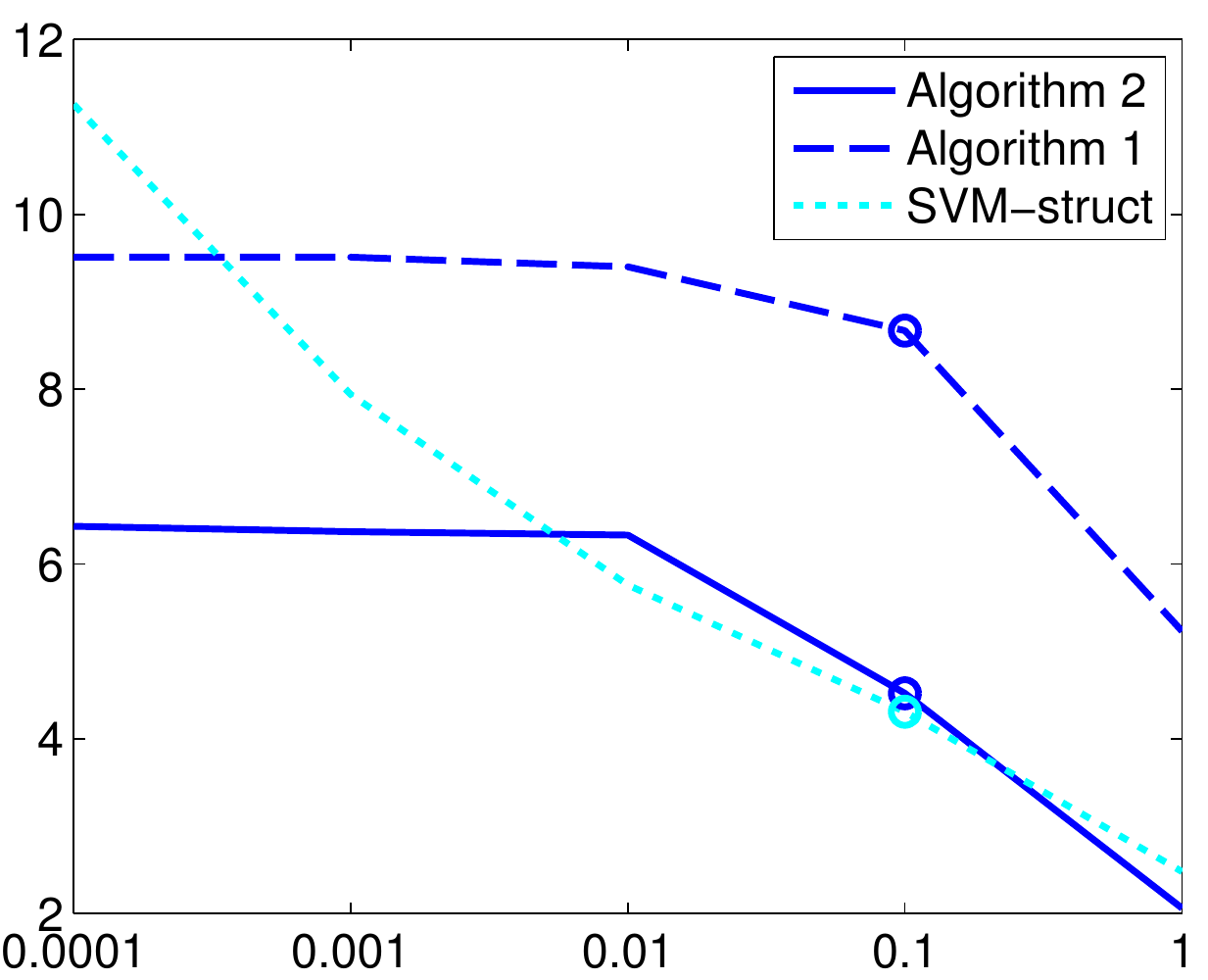}}
	\hfill
	\subfloat[$L/K=5$ (average time vs $\alpha$)]{ \includegraphics[width=0.45\textwidth]{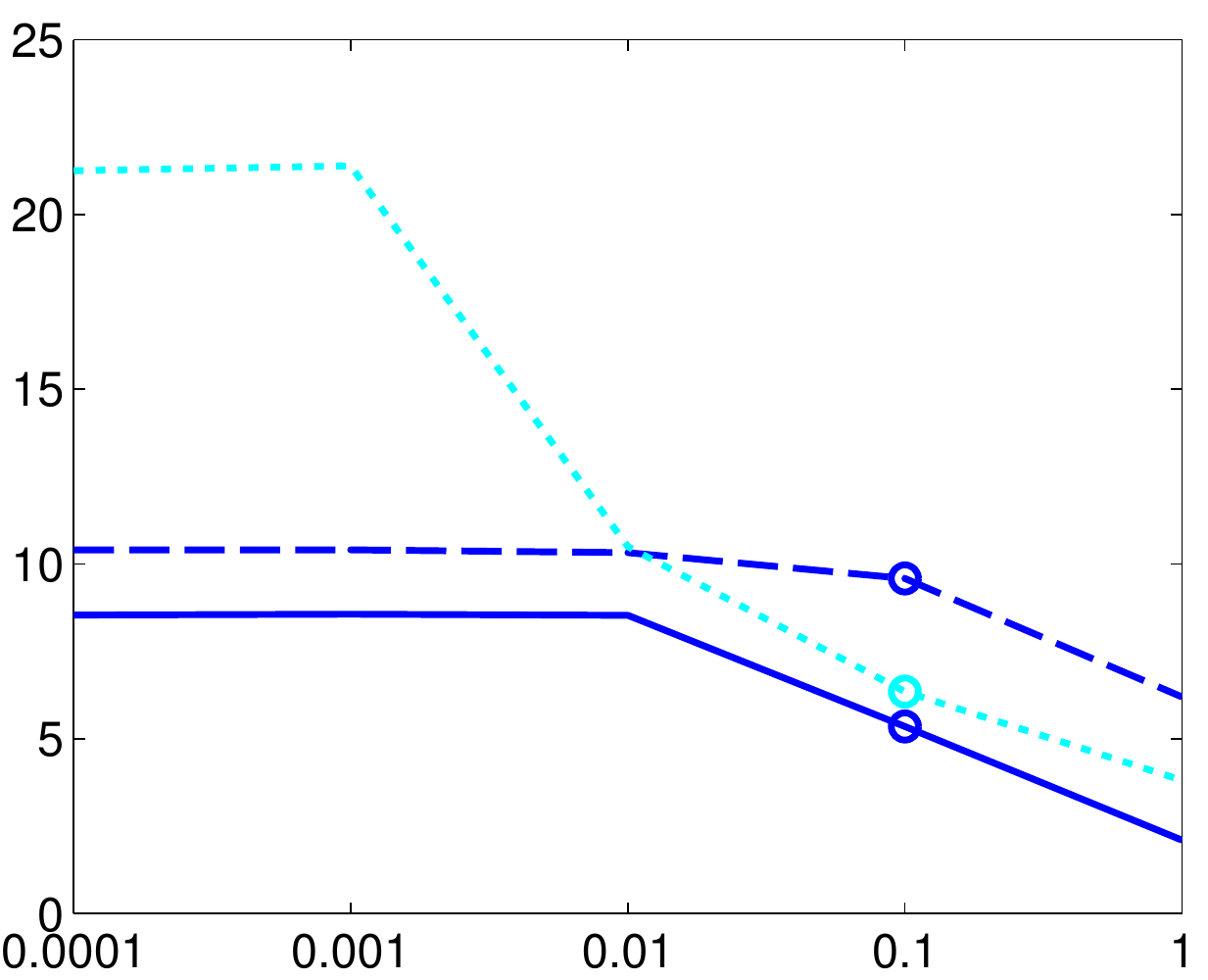}}
	\hfill
	\subfloat[$L/K=10$ (average time vs $\alpha$)]{ \includegraphics[width=0.45\textwidth]{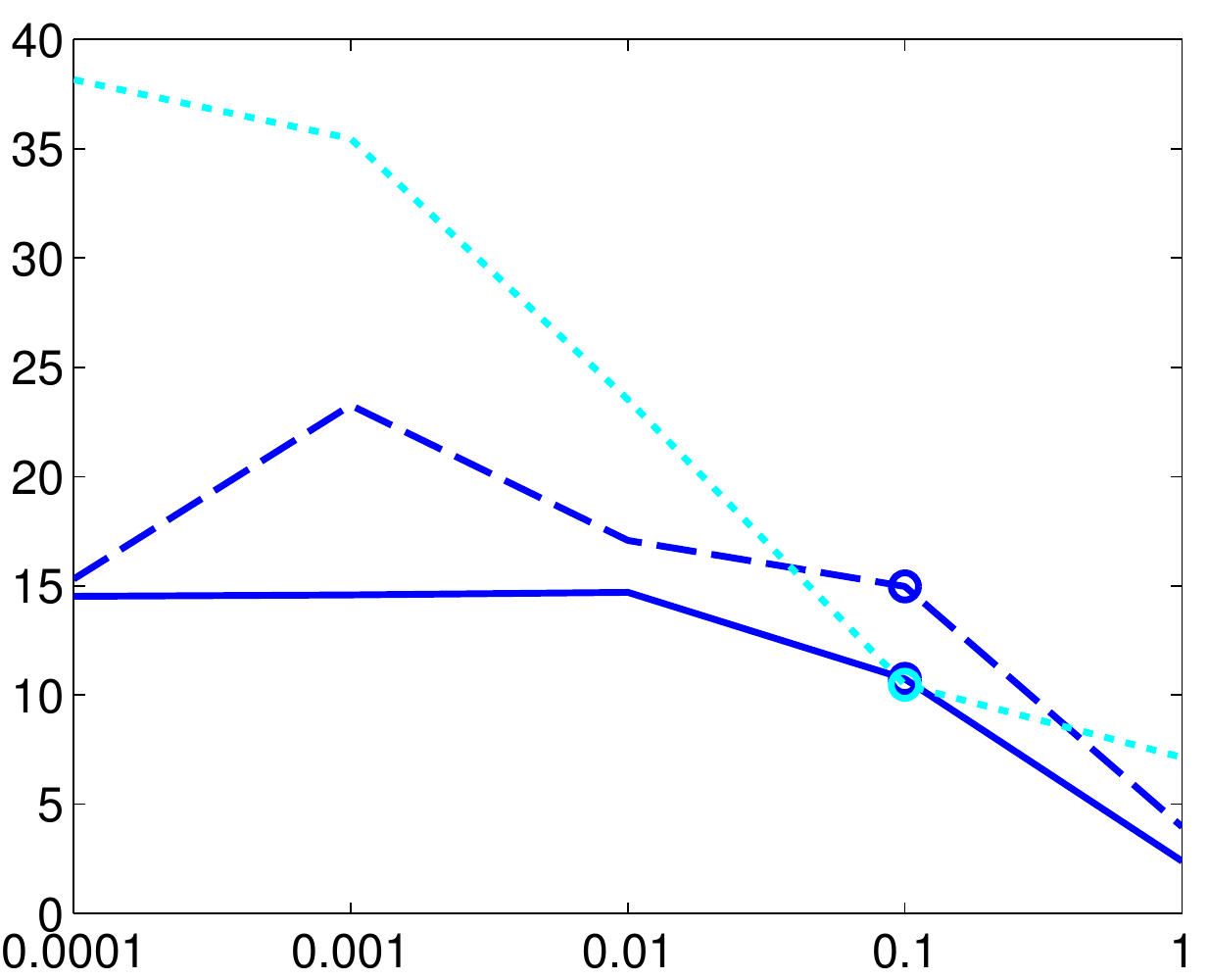}}
	\hfill
	\subfloat[$L/K=50$ (average time vs $\alpha$)]{ \includegraphics[width=0.45\textwidth]{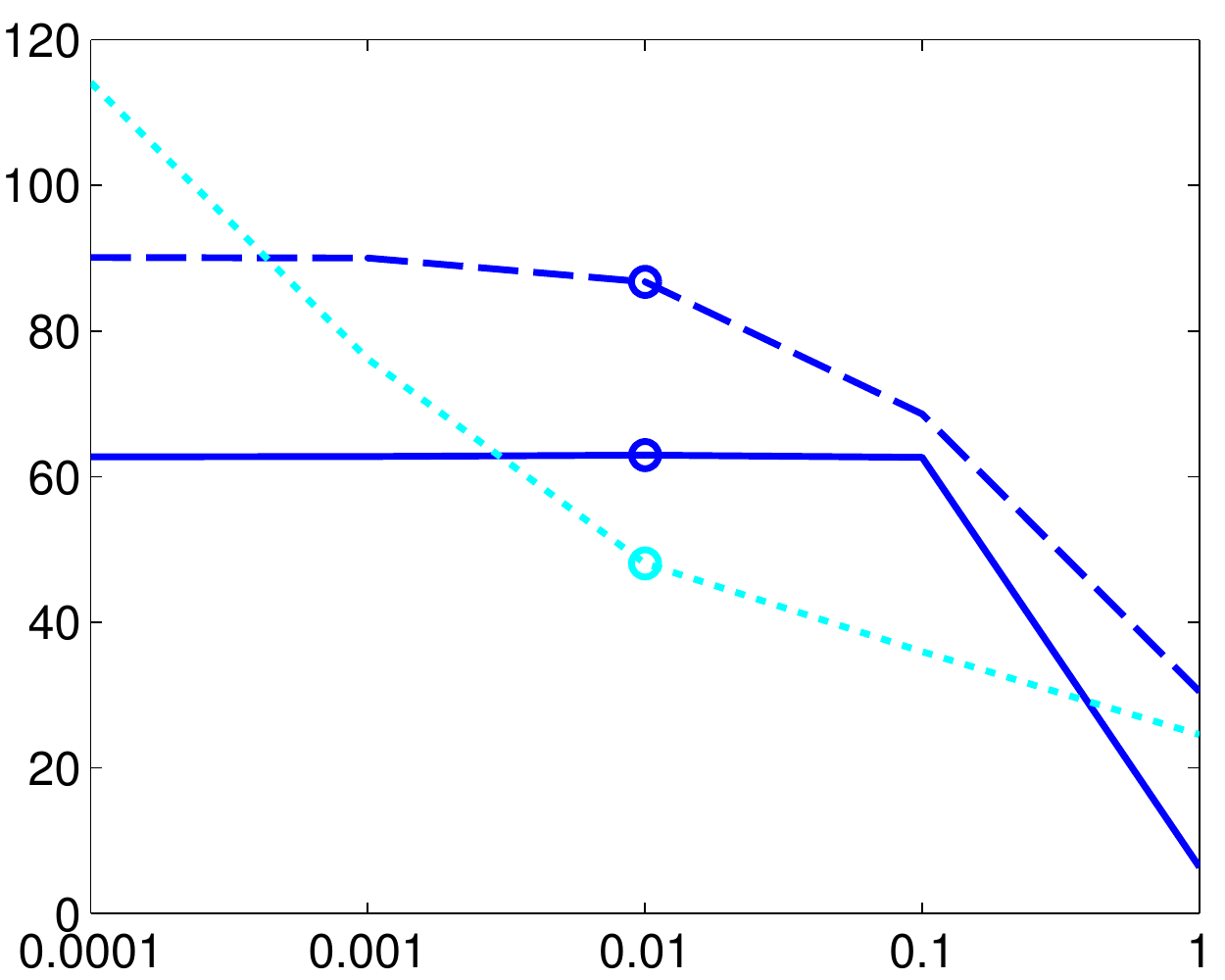}}
	\hfill
	\subfloat[$L/K=100$ (average time vs $\alpha$)]{ \includegraphics[width=0.45\textwidth]{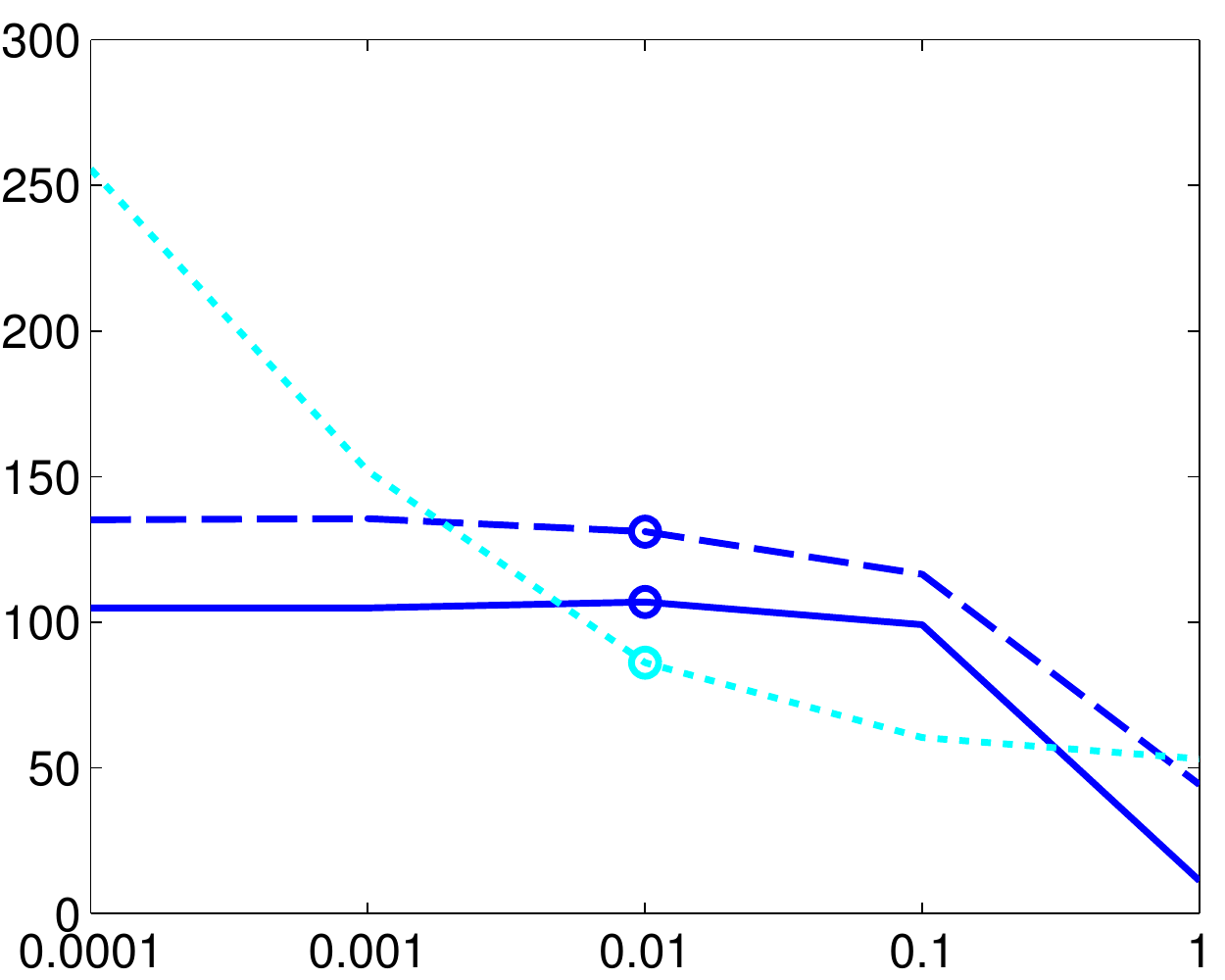}}
	\hfill
	\subfloat[$L/K=500$ (average time vs $\alpha$)]{ \includegraphics[width=0.45\textwidth]{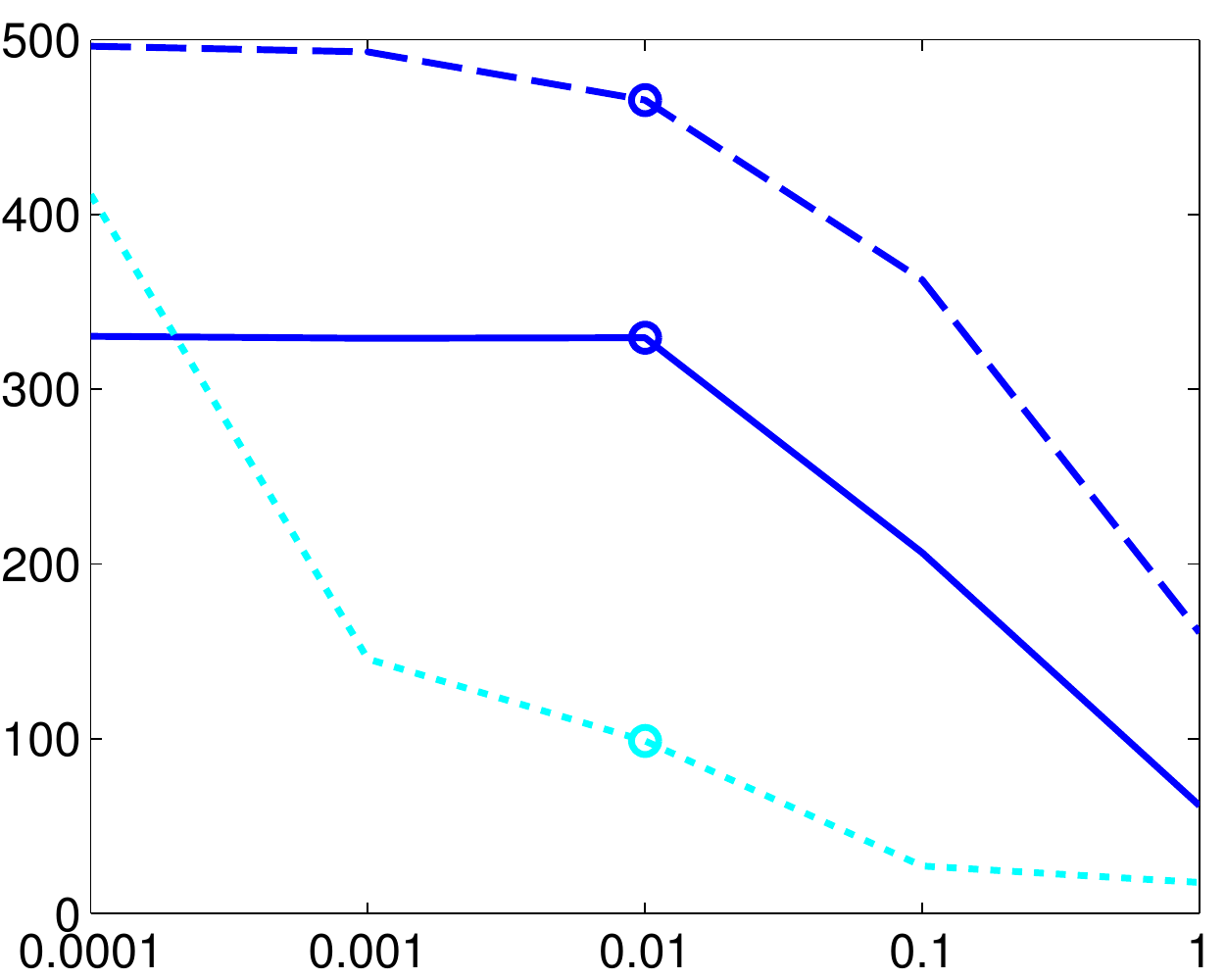}}

	\caption{Results on MNIST database with a quadratic regularization. The plots show the execution times obtained with a stopping criterion of $10^{-3}$ for some values of $\alpha$. The circles mark the values of $\alpha$ yielding the best classification accuracy.}
	\label{fig:times_L2}
\end{figure*}

\section{Conclusions}
We have proposed two efficient algorithms for learning a sparse multiclass SVM. Our approach makes it possible to minimize a criterion involving the multiclass hinge loss and a sparsity-inducing regularization. In the literature, such a criterion is typically approximated by replacing the hinge loss with a smooth penalty, such as the quadratic hinge loss or the logistic loss. In this paper, we have provided two solutions that directly deal with the hinge loss: one addressing the regularized formulation and the other one adapted to the constrained formulation. The performance of the proposed solutions have been evaluated over three databases in scenarios with a few training data. The results show that the use of the hinge loss, rather than an approximation, leads to a slightly better classification accuracy and tends to make the method more robust w.r.t. the choice of the regularization parameter, while the proposed algorithms are often faster than state-of-the-art solutions.

\bibliographystyle{IEEEbib}
\bibliography{biblio}

\end{document}